\documentclass{article}
\usepackage{ijcai26}

\usepackage{times}
\usepackage{soul}
\usepackage{url}
\usepackage[hidelinks]{hyperref}
\usepackage[utf8]{inputenc}
\usepackage[small]{caption}
\usepackage{graphicx}
\usepackage{amsmath}
\usepackage{amsthm}
\usepackage{booktabs}
\usepackage{algorithm}
\usepackage{algorithmic}
\usepackage[switch]{lineno}

\usepackage{natbib}

\usepackage{graphicx} %
\usepackage{amsmath}
\usepackage{amsthm}
\usepackage{amssymb}
\usepackage{amsfonts}
\usepackage{nicefrac}
\usepackage{microtype}
\usepackage[dvipsnames, x11names]{xcolor}
\usepackage{booktabs} %
\usepackage{tabularx}
\usepackage{placeins}
\usepackage[pro]{fontawesome5}
\usepackage[inline]{enumitem}
\usepackage[hidelinks]{hyperref} %
\hypersetup{
    colorlinks=True,
    allcolors=Blue,
    }
\usepackage[capitalize]{cleveref} %
\usepackage{xspace}
\usepackage{csquotes}
\usepackage{tikz}
\usetikzlibrary{arrows, automata, positioning}
\usepackage{mdframed}

\usepackage{comment}

\usepackage{silence}
\usepackage{scalerel}  %
\usepackage{subcaption}
\renewcommand{\thesubfigure}{~(\text{%
  \ifnum\value{subfigure}=1 left\else right\fi})}

\usepackage{thm-restate}

\usepackage{mathtools}

\usepackage[most]{tcolorbox}
\newtcolorbox{leftvrule}[1][]{colback=white,  boxrule=0pt, boxsep=0pt, breakable, enhanced jigsaw, borderline west={1.5pt}{0pt}{black},
before skip=5pt,after skip=5pt,
#1}

\Crefname{equation}{Equation}{Equations}
\crefname{equation}{Equation}{Equations}
\Crefname{figure}{Figure}{Figures}
\crefname{figure}{Figure}{Figures}
\Crefname{remark}{Remark}{Remarks}
\crefname{remark}{Remark}{Remarks}
\Crefname{section}{Section}{Sections}
\crefname{section}{Section}{Sections}
\Crefname{subsection}{Section}{Sections}
\crefname{subsection}{Section}{Sections}
\Crefname{appendix}{Appendix}{Appendices}
\crefname{appendix}{Appendix}{Appendices}
\Crefname{example}{Example}{Examples}
\crefname{example}{Example}{Examples}
\Crefname{definition}{Definition}{Definitions}
\crefname{definition}{Definition}{Definitions}
\Crefname{lemma}{Lemma}{Lemmata}
\crefname{lemma}{Lemma}{Lemmata}
\Crefname{theorem}{Theorem}{Theorems}
\crefname{theorem}{Theorem}{Theorems}

\theoremstyle{definition} %

\newtheorem{definition}{Definition}
\newtheorem{remark}{Remark}
\newtheorem{theorem}{Theorem}
\newtheorem{proposition}{Proposition}
\newtheorem{lemma}{Lemma}

\newtheorem{example}{Example}

\newcommand{\abs}[1]{\lvert #1 \rvert}

\newcommand{\mathsymbol}[2]{ \newcommand{#1}{\ensuremath{#2}\xspace} }

\newcommand{\pr}{\mathbb{P}}
\newcommand{\supp}{\text{supp}}

\newcommand{\T}[0]{{T}}
\newcommand{\mis}[0]{\bot}
\newcommand{\Dist}[0]{\Delta}
\mathsymbol{\mdp}{\mathcal{P}}
\mathsymbol{\mdphat}{\MDPhat}
\mathsymbol{\dataset}{\mathcal{D}}
\mathsymbol{\datasetsize}{|\dataset|}
\newcommand{\init}{b_0}%
\newcommand{\Rfunc}{f_R}

\newcommand{\Z}{Z}    
\newcommand{\Zmiss}{\Z}
\newcommand{\reward}{\varrho}
\newcommand{\misstup}{\ensuremath{(S, A, \T, \init, \reward, Z, M, \gamma)}}
\newcommand{\beliefset}{\mathcal{B}}
\newcommand{\smallextra}{\kappa}
\newcommand{\indicatorfunction}{\boldsymbol{1}}

\newcommand{\scalesymbol}[2]{\mathchoice
  {\scalebox{#1}{$\displaystyle{#2}$}}
  {\scalebox{#1}{$\textstyle{#2}$}}
  {\scalebox{#1}{$\scriptstyle{#2}$}}
  {\scalebox{#1}{$\scriptscriptstyle{#2}$}}
}

\newcommand{\Mhat}{\scalesymbol{0.9}{\widehat{M}}}
\newcommand{\MDPhat}{\scalesymbol{0.85}{\widehat{\mdp}}}

\newcommand{\Ialw}{\scalesymbol{0.9}{I_\text{always}}}
\newcommand{\Imis}{\scalesymbol{0.9}{I_\text{mis}}}
\newcommand{\Ihatalw}{\scalesymbol{0.9}{\hat{I}_\text{always}}}
\newcommand{\Ihatmis}{\scalesymbol{0.9}{\hat{I}_\text{mis}}}

\mathsymbol{\sta}{x}
\mathsymbol{\stb}{y}

\newcommand{\rand}[1]{\boldsymbol{#1}}

\mathsymbol{\traj}{\upsilon}
\mathsymbol{\trajs}{\Upsilon}
\mathsymbol{\histories}{\mathcal{H}}

\mathsymbol{\sie}{\mathrm{E}}

\definecolor{ColorSm}{HTML}{9773A6}
\definecolor{ColorSo}{HTML}{81B367}
\definecolor{ColorR}{HTML}{6C8EBF}
\definecolor{ColorZ}{HTML}{D6B656}

\newcommand{\Node}[2]{
    \tikz[scale=0.5, transform shape,  baseline={(0,-0.75ex)}]{%
        \filldraw[fill=white, draw=black, line width=0.2mm] (0,0) circle (1.9ex);
        \node at (0,0) {#2};
    }
}

\newcommand{\Znode}[0]{\!\Node{ColorZ}{$Z$}\!}
\newcommand{\Rnode}[0]{\!\Node{ColorR}{$R$}\!}
\newcommand{\Snode}[0]{\!\Node{white}{$S$}\!}

\newcommand{\eps}{\ensuremath{\varepsilon}}

\usepackage{xargs}
\usepackage{silence}

\AtBeginDocument{\renewcommand{\ref}[1]{\PackageError{goodpractices}{do not use ref, replace it by cref}{do not use ref, replace it by cref}}}

\mathsymbol{\piopt}{\pi^{*}}
\mathsymbol{\pirand}{\pi^{\text{rand}}}
\mathsymbol{\pihat}{\hat{\pi}}
\mathsymbol{\Mguess}{\text{prior}}
\mathsymbol{\piguess}{\pi^{\Mguess}}
\mathsymbol{\piPPO}{\pi^{\text{PPO}}}
\mathsymbol{\piPPOmemory}{\pi^{\text{PPO+Memory}}}
\mathsymbol{\piPOMCP}{\pi^{\text{POMCP}}}

\newcommand{\missMDP}{miss-MDP\xspace}
\newcommand{\MissMDP}{Miss-MDP\xspace}
\newcommand{\missMDPs}{miss-MDPs\xspace}
\newcommand{\MissMDPs}{Miss-MDPs\xspace}

\newcommand{\epsA}{\varepsilon_{M}}
\newcommand{\epsPi}{\varepsilon_{\pi}}
\newcommand{\epsG}{\varepsilon_{\gamma}}

\newcommand{\alg}[1]{\texttt{#1}\xspace}
\newcommand{\environment}[1]{\emph{#1}\xspace}
\newcommand{\process}[1]{\emph{#1}\xspace}
\newcommand{\envprocess}[2]{{$#1_{#2}$}\xspace}

\newcommand{\AMCAR}{\alg{AMCAR}}
\newcommand{\AsMAR}{\alg{AsMAR}}
\newcommand{\AIMI}{\alg{AIMI}}
\newcommand{\ICU}{\environment{ICU}}
\newcommand{\Predator}{\environment{Predator}}
\newcommand{\MCAR}{\process{MCAR}}
\newcommand{\simpleMAR}{\process{sMAR}}
\newcommand{\MNARid}{\process{MNAR (id.)}}
\newcommand{\MNARunid}{\process{MNAR (unid.)}}

\newcommand{\PredatorMCAR}{\envprocess{\Predator}{\MCAR}}
\newcommand{\PredatorSimpleMAR}{\envprocess{\Predator}{\simpleMAR}}
\newcommand{\PredatorMNARunid}{\envprocess{\Predator}{\MNARunid}}

\newcommand{\ICUSimpleMAR}{\envprocess{\ICU}{\simpleMAR}}
\newcommand{\ICUMNARid}{\envprocess{\ICU}{\MNARid}}
\newcommand{\ICUMNARunid}{\envprocess{\ICU}{\MNARunid}}

\newlist{questionenum}{enumerate}{1}
\setlist[questionenum]{label=\textbf{Q\arabic*.}, ref=Q\arabic*, leftmargin=2\parindent, nosep}
\Crefname{question}{}{}
\crefname{question}{}{}
\crefalias{questionenumi}{question}

\newcommand{\topleftframe}[3][0.6pt]{%
  \begin{tikzpicture}[baseline=(img.base)]
    \node[inner sep=0pt, outer sep=0pt, anchor=base] (img) {#3};
    \begin{scope}[overlay]
      \draw[line width=#1, draw=#2] (img.north west) -- (img.north east);
      \draw[line width=#1, draw=#2] (img.north west) -- (img.south west);
    \end{scope}
  \end{tikzpicture}%
}

\newcommand{\rebuttal}[1]{#1}

\newif\ifseparate
\newif\ifappendix

\separatefalse
\appendixfalse

\ifdefined\buildmain
  \separatetrue
  \appendixfalse
\fi
\ifdefined\buildappendix
  \separatetrue
  \appendixtrue
\fi

\ifseparate
  \usepackage{xr}            %
\fi

\title{Missingness-MDPs: Bridging the Theory of Missing Data and POMDPs}

\author{
Joshua Wendland$^1$
\and
Markel Zubia$^1$\and
Roman Andriushchenko$^{2}$\and
Maris F. L. Galesloot$^3$\and\\
Milan Ceska$^2$\and
Henrik von Kleist$^4$\and
Thiago D. Sim\~ao$^5$\and
Maximilian Weininger$^1$\And
Nils Jansen$^{1,3}$
\\
\affiliations
$^1$Ruhr University Bochum\\
$^2$Brno University of Technology\\
$^3$Radboud University Nijmegen\\
$^4$Harvard University\\
$^5$Eindhoven University of Technology\\
\emails
\{joshua.wendland, markel.zubia\}@ruhr-uni-bochum.de
}

\begin{document}

\maketitle

\begin{abstract}
We introduce \emph{missingness-MDPs}~(\missMDPs), 
a novel subclass of partially observable Markov decision processes~(POMDPs) that incorporates the theory of missing data.
A \missMDP is a POMDP whose observation function is a \emph{missingness function}, specifying the probability that individual state features are \emph{missing} (i.e.,  unobserved) at a time step.
The literature distinguishes three canonical missingness types: missing (1) completely at random (MCAR), (2) at random (MAR), and (3) not at random (MNAR).
Our planning problem is to compute near-optimal policies for a \missMDP with an \emph{unknown} missingness function, given a dataset of action–observation trajectories.
Achieving such optimality guarantees for policies requires learning the missingness function from data, which is infeasible for general POMDPs.
To overcome this challenge, we exploit the structural properties of different missingness types to derive probably approximately correct (PAC) algorithms for learning the missingness function.
These algorithms yield an approximate but fully specified \missMDP that we solve using off-the-shelf planning methods.
We prove that, with high probability, the resulting policies are 
$\varepsilon$-optimal in the true \missMDP.
Empirical results confirm the theory and demonstrate superior performance of our approach %
over two model-free POMDPs methods.
\end{abstract}

\section{Introduction}

Markov decision processes \citep[MDPs;][]{DBLP:books/wi/Puterman94} capture sequential decision-making under uncertainty. 
A typical assumption is that all \emph{state features} are fully observable at all times.
In practice, however, features may be \emph{missing}, e.g.\ due to sensor failures, 
so decisions must be made from incomplete information.
For instance, a medical doctor may diagnose a patient based on recorded measurements such as heart rate and temperature that are potentially incomplete.

Our problem is to compute optimal policies for MDPs with dynamically missing state features.
While the MDP structure and historical data are given, the process governing missingness is unknown.
Partially observable Markov decision processes \citep[POMDPs;][]{astrom1965optimal} naturally capture this setting, with an observation function specifying the probability of receiving partial state information, of which missing features are a special case.
However, the observation function is unknown and must, in principle, be learned from action–observation histories.
Yet, learning the observation function of a POMDP is notoriously challenging: 
In general, it is unidentifiable from histories alone~\citep{allman2009identifiability,finesso1990consistent,gilbert1959identifiability,tune2013hidden}, 
and even without requiring identifiability (i.e.\ merely seeking a model consistent with the data), learning remains both statistically~\citep{krishnamurthy2016pac,xiong2022sublinear,golowich2022learning} and computationally intractable~\citep{terwijn2002learnability,mossel2005learning}.

Our core contribution is exploiting the theory of \emph{missing data} \citep{schafer_missing_2002,buuren_flexible_2018,little_statistical_2019} to define suitable identifiability assumptions for POMDPs. 
We introduce  \emph{missingness-MDPs} (\missMDPs) as a proper subclass of POMDPs.
In \missMDPs, the observation function is a \emph{missingness function} that generates either perfect or missing information for each state feature.
We classify this function as \emph{missing completely at random}~(MCAR), \emph{missing at random}~(MAR), or \emph{missing not at random}~(MNAR), following~\citet{rubin_inference_1976}, detailed in \cref{sec:missingness}.

We refine our problem statement: 
Given a \missMDP with an \emph{unknown} missingness function and a dataset of actions and observations, the goal is to compute a policy 
that maximizes the expected reward.
We assume that the underlying MDP and, thus, the transition function are known.
In the doctor-patient example, the transition function captures the change in the patient's health status \citep{komorowski_artificial_2018}, see \cref{fig:missingness-MDP-visualization}.
Similarly, in a robot navigation task, the probability of transitioning to a subsequent state may be known, while sensor failures can still introduce missingness and, therefore, partial observability.
To obtain guarantees on the result, we approximate the unknown missingness function
from the dataset, and thereby the original \missMDP.
For this approximate \missMDP, we compute an approximately optimal policy through off-the-shelf POMDP solvers such as SARSOP~\citep{kurniawati2008sarsop}.
Missingness functions are \emph{not learnable in general}~\citep{bhattacharya_identification_2020},
yet we identify and, subsequently, focus on missingness functions that are tractable to learn.

In summary, our contributions are:
\begin{enumerate}[topsep=2pt, itemsep=1.5pt, parsep=1.5pt, partopsep=0.5pt, leftmargin=20pt, labelsep=0.5em, labelindent=0pt, labelwidth=!]
    \item We introduce \missMDPs, which integrate and define the semantics of missing data in a proper subclass of the more general POMDP framework (\cref{sec:missingness}).
    \item We prove that state beliefs do not always depend on the missingness function's probabilities (\cref{rem:ignorability}), similar to \emph{ignorability} of missing data~\citep{little_statistical_2019}.
    \item We present \emph{probably approximately correct}~(PAC) learning algorithms for tractable subclasses of the three missingness types.
    (\cref{sec:algorithm-AIMI,sec:algoritm-AMCAR-AsMAR}).
    \item We prove that by using these algorithms, we approximate the $\eps$-optimal policy for the \missMDP under the correct assumption on the missingness function (\cref{sec:alg-policy}).
\end{enumerate}
Our empirical evaluation (\cref{sec:exp}) confirms our theoretical results and highlights the practical advantages of our approach: 
(1) When using datasets of moderate size to learn the missingness function, the performance of the resulting policies converges to that of the optimal policy; and
(2) our policies show superior performance in comparison to POMCP~\citep{silver2010monte} and PPO~\citep{DBLP:journals/corr/SchulmanWDRK17}, which use the direct POMDP formulation of the missingness problem.

\begin{figure}[tb]
    \centering
    \includegraphics[width=\linewidth, clip, trim= 30 7 50 8]{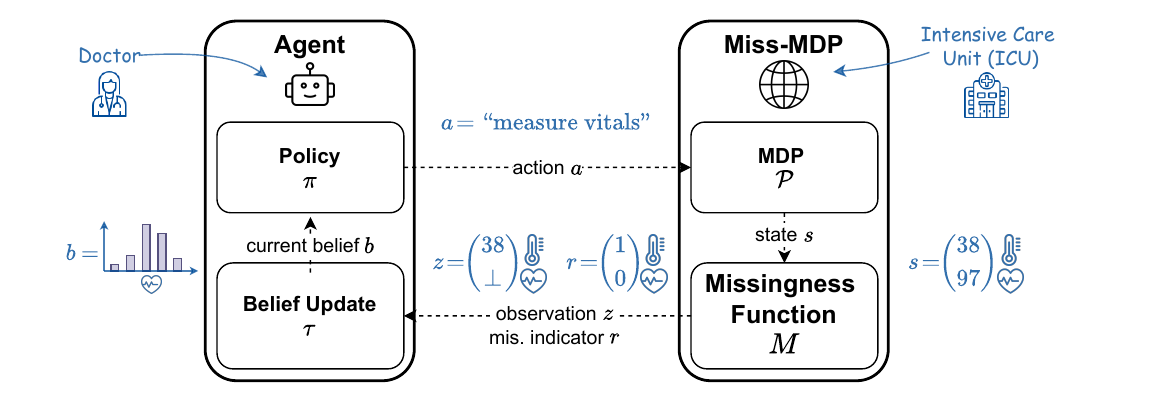}
    \vspace{-2ex}
    \definecolor{exampleblue}{HTML}{326FAD}
    \caption{
    A doctor-treating-patient example (\textcolor{exampleblue}{blue annotations}) of an agent interacting with a \missMDP.
    The missingness function causes the heart rate feature of the state to go missing, indicated as $\mis$ in the observation. The missingness indicator evaluates to 0 for missing features and to 1 otherwise.}
    \label{fig:missingness-MDP-visualization}
\end{figure}

\subsection*{Related work}

Our work builds on a rich literature in missing data analysis, see e.g. \citep{anastasios_a_semiparametric_2006,little_statistical_2019}. 
Classical assumptions such as MCAR, MAR, and MNAR provide high-level categories. 
More refined tools, such as missingness graphs, allow one to encode assumptions about the missingness in a structured way \citep{mohan_graphical_2013, shpitser_missing_2015}, leading to highly specific learnability results~\citep{bhattacharya_identification_2020, nabi_full_2020}.
Our setting departs from the standard missing data paradigm in several important aspects. 
In particular, the concept of missingness is embedded in the setting of solving POMDPs, enabling a more principled understanding of missingness in the context of sequential decision-making under uncertainty.

Prior work on MDPs with missing observations has focused mainly on reinforcement learning, where missing data is treated as incidental rather than explicitly modeled.
Both full observations~\citep{DBLP:conf/nips/Chen0PW23} or individual features may be missing~\citep{shim_joint_2018, yoon_asac_2019, bock_superhuman_2022}.
Some approaches pre-process missingness in observations for RL agents~\citep{wang_robust_2019}, while others adopt model-based methods, often restricted to simpler settings such as MCAR~\citep{futoma_popcorn_2020}.
Planning approaches typically neglect distinctions between MCAR, MAR, and MNAR \citep{liu_machine_2022,DBLP:conf/ic-nc/YamaguchiFO20,futoma_popcorn_2020}.
Another line of work combines deep learning with POMDP solvers, but without explicitly modeling the missingness process~\citep{liu_machine_2022}.
More principled \emph{imputation strategies}—such as Bayesian multiple imputation~\citep{lizotte2008missing} and expectation-maximization~\citep{DBLP:conf/ic-nc/YamaguchiFO20}—estimate a distribution over missing values.
In contrast to imputation, our approach directly learns the missingness function and provides PAC guarantees on the resulting policy.
To our knowledge, no existing work bridges missingness and POMDPs to (1) learn the missingness function with statistical guarantees, and (2) leverage it to guarantee the optimality of the resulting policies.

\section{Preliminaries}
\label{sec:preliminaries}
A \emph{probability distribution over a finite
set $X$} is a function $\mu \colon X \to [0, 1]$ with $\sum_{x \in X} \mu(x) = 1$.
The set of such distributions is $\Dist(X)$.
Writing $\mu = \{x_1 \mapsto p_1, \dots, x_k \mapsto p_k\}$ indicates that $\mu(x_1) = p_1$ and so on.
The \emph{support} of distribution $\mu \in \Dist(X)$ is $\supp(\mu) = \{x \in X \mid \mu(x) \neq 0\}$.
Sampling a random variable $\rand{x}$ from $\mu$ is denoted by $\rand{x} \sim \mu$.
Given $\sigma\colon X \to \Dist(Y)$, we let $\sigma(y \mid x) := \sigma(x)(y)$.
The indicator function $\rand{1}_\varphi$ returns $1$ if predicate $\varphi$ holds and $0$~otherwise.

\begin{definition}[POMDPs]
    A \emph{partially observable Markov decision process} is a tuple $\mdp = (S, A, \T, \init, \reward, Z, O, \gamma)$
    with finite factored \emph{state space} $S = \times_{i = 1, \dots, n} \, S_i$  and the set of \emph{feature indices} $I = \{1,\ldots,n\}$,  finite \emph{action space} $A$, \emph{transition function} $\T \colon S \times A \to \Dist(S)$,
    \emph{initial state distribution} $\init \in \Dist(S)$,  \emph{reward function} $\reward \colon S \times A \to \mathbb{R}$,  finite factored \emph{observation space} $Z = \times_{i = 1, \dots m} Z_i$, \emph{observation function} $O \colon S \to \Dist(Z)$,
    and \emph{discount factor} $\gamma \in [0, 1)$.
\end{definition}

The observation function is action-independent without loss of generality, as we may augment the state space to incorporate the last performed action \citep{ChatterjeeCGK16}.

A \emph{trajectory} in a POMDP $\mdp$
is a sequence of states, observations, and actions.
A \emph{history} $h = \bigl(z^{(0)},a^{(0)}, z^{(1)},a^{(1)},\dots\bigr) \in \histories \subseteq (Z \times A)^*$ is the observable fragment of a trajectory, i.e., a sequence of observations and actions.
A history can be summarized by a \emph{sufficient statistic} known as a \emph{belief} $b\in\beliefset\subseteq \Dist(S)$, a probability distribution over underlying states induced by a history $h\in\histories$.
The \emph{belief update} $\tau \colon \beliefset \times A \times \Z \to \beliefset$ computes a \emph{successor belief}~$b'$ 
via Bayes' rule~\citep{Spaan2012}.

A \emph{policy} $\pi \colon \beliefset \to \Dist(A) \in \Pi$ maps beliefs to probability distributions over actions.
The \emph{objective} is to 
find a policy \(\pi\in\Pi\) that 
maximizes its \emph{value}, representing the infinite-horizon expected cumulative discounted reward:
\(
V_{\mdp}(\pi) = \mathbb{E}^\pi\left[ \sum_{t=0}^\infty \gamma^t \reward(s^{(t)}, a^{(t)}) \right].
\)
As the decision problem of finding the optimal policy is undecidable in general~\citep{DBLP:journals/ai/MadaniHC03},
we focus on \eps-optimal policies~\citep{hauskrecht2000value}.

\section{Missingness in MDPs}\label{sec:missingness}

This section introduces \emph{missingness-MDPs}, the different types of missingness functions and ways to visualize them.

\begin{definition}[\MissMDP]
    A missingness-MDP is a tuple $\mdp =\misstup$, where 
    $S$, $A$, $\T$, $\init$, $\reward$, and $\gamma$ are as in a POMDP, 
    the finite \emph{observation space} is $\Zmiss = \times_{i \in I} (S_i \cup \{\mis\})$, with $\mis$ denoting \emph{missing information},
    and function $M\colon S \rightarrow \Dist(\Zmiss)$ is the \textit{missingness function} such that $\forall s \in S, \forall z \in \supp(M(s)), \forall i \in I$ either $z_i= s_i$ or $z_i = \mis$.
\end{definition}
The state space $S$ and observation space $\Zmiss$ share the feature indices~$I$, and we have that $\Zmiss \supsetneq S$ as some features can go \emph{missing} in $\Zmiss$, being replaced by the symbol $\mis$.
This process of \enquote{poking holes} is governed by the stochastic missingness function~$M$.
While $M$ may take actions into account, we use an action-independent $M$ w.l.o.g (see \cref{sec:preliminaries}).

\paragraph{Missingness indicators.} \label{sec:missingness-indicator-definition}
Missingness functions can equivalently be described as a map to vectors of \emph{missingness indicators}~\citep{mohan_graphical_2013}, i.e.\ $M\colon S\to \Delta(R)$, where $R = \{0,1\}^n$.
A vector $r \in R$ has $r_i = 0$ if feature $i$ is missing ($z_i=\mis$), and otherwise  $r_i=1$.
The function $\Rfunc \colon Z \rightarrow R$ maps observations to their missingness indicators.

\begin{example}
\label{ex:missing-pomdp}
    Let $\mdp$ be a \missMDP with $S = \{\sta,\stb\}^2$, $\Zmiss = \{\sta,\stb,\mis\}^2$,
    and missingness function~$M$ defined as:
    $M((s_1, s_2)) = \{(s_1, s_2) \mapsto 0.5, (s_1, \bot) \mapsto 0.5\}$.
    We have $\Rfunc((\stb,\sta)) = (1,1)$ and $\Rfunc((\stb,\bot)) = (1,0)$.
    Visiting state $(\stb,\sta)$ yields either $(\stb,\sta)$ or $(\stb,\bot)$, each with probability~0.5.
\end{example}

\begin{mdframed}[innertopmargin=4pt,
    innerbottommargin=4pt, nobreak=true]
\paragraph{Problem statement.}
We are given a \missMDP $\mdp$ with an \emph{unknown} missingness function $M$, a required precision $\eps > 0$ and confidence threshold $\delta > 0$, and a dataset $\dataset = (h_1,\dots,h_k)$ of 
$k$ histories $h_i \in \histories$ collected from $\mdp$ under an unknown but fair 
policy $\pi_b$ (i.e.\ it has positive probability to visit all reachable states).
The goal is to (1) approximate the missingness function $\Mhat\approx M$ for all reachable states and (2) use it to compute a policy $\pi^*\in \Pi$ such that with probability at least $\delta$, we have
$\sup_{\pi} (V_{\mdp}(\pi)) - V_{{\mdp}}(\pi^*) \leq \varepsilon$.
\end{mdframed}
The resulting policy $\pi^*$ comes with a PAC guarantee~\citep{valiant1984theory,weininger_PAC_CAV}: In \emph{finite time} and with probability at least $\delta$, we return a policy that is $\eps$-optimal.

\paragraph{Advantages of \missMDPs over POMDPs.}
For computing $\varepsilon$-optimal policies in POMDPs, knowledge of the observation function is required.
However, for arbitrary POMDPs, learning the observation function from a dataset of histories is impossible~\citep{gilbert1959identifiability, tune2013hidden}. 
In contrast, \missMDPs exhibit structure in the missingness function~$M$ that, for certain types of missingness function, allows for learning $M$ from histories (see also the paragraph \enquote{Missingness types in focus} in \cref{sec:algos}).
Consequently, \missMDPs are a \emph{proper and more feasible subclass} of POMDPs.

\subsection{Types of missingness functions}
There are three major types of missingness functions~\citep{rubin_inference_1976}: 
\emph{missing completely at random} (MCAR), \emph{missing at random} (MAR) and \emph{missing not at random} (MNAR).
We define these in the context of \missMDPs. %
This way, we connect the theory of missing data to POMDPs, which allows us to identify cases in which $M$ can be learned from histories.

The simplest type is MCAR, where the probability of features going missing is independent of the \emph{feature values} of the state.
In the doctor-diagnosing example (\cref{fig:missingness-MDP-visualization}), the temperature feature may be missing due to a loosely attached thermometer.
The \missMDP in \cref{ex:missing-pomdp} is MCAR.
\begin{definition}[\textbf{MCAR}]
    The missingness function $M \colon S \rightarrow \Dist(Z)$ of a \missMDP~$\mdp$ is MCAR if and only if
    $\forall r \in R$, $\exists p_r \in [0, 1]$, $\forall s \in S$,
    $\pr\left(\Rfunc(\rand z) = r \mid \rand z \sim M(s)\right) = p_r$. 
\end{definition}

\paragraph{Admittability and $\Ialw$.}
We introduce a notion of admittability that indicates whether an observation $z$ could originate from a state $s$.
We say that $z$ is \emph{admittable} by $s$, denoted $z \preceq s$, if and only if $\forall i \in I$, $z_i = \bot$ or $z_i = s_i$.
In \cref{ex:missing-pomdp}, we have $(b,\mis) \preceq (b,a)$ and $(b,a) \preceq (b,a)$ but $(a,\mis) \npreceq (b,a)$.
Further, $\Ialw = \{i \in I \mid \forall s' \in S\colon \pr(\rand z_i = \bot \mid \rand z \sim M(s')) = 0\}\subseteq I$ is the set of indices of always observed features, and $\Imis = I \setminus \Ialw$ is its complement.

Missingness probabilities in MAR depend only on observed state features -- \rebuttal{in the example of}~\cref{fig:missingness-MDP-visualization}, the observed temperature influences the missingness of the heart rate feature.
We distinguish two MAR variants: a restricted one we call \emph{simple MAR}~\citep{Mohan_graphical_2021}, and the general one~\citep{rubin_inference_1976}.
For simple MAR, a missingness probability is only influenced by the observable features that \emph{never go missing}, i.e., by $z_i$ for $i \in \Ialw$.
For MAR, a missingness probability is only influenced by the non-missing features of a given \emph{observation}, including features that may go missing.
Any MCAR missingness function is also (simple) MAR.

\begin{definition}[\textbf{(Simple) MAR}]
    The missingness function $M \colon S \rightarrow \Dist(Z)$ of a \missMDP~$\mdp$ is:
    \begin{itemize}[topsep=0pt, itemsep=1.5pt, parsep=1.5pt, partopsep=0.5pt, leftmargin=12pt, labelsep=0.5em, labelindent=0pt, labelwidth=!]
    \item \textbf{Simple MAR} iff
    for all $s, s' \,{\in}\, S$ that agree on always-observed features (i.e.\ $\forall i \,{\in}\, \Ialw$, $s_i \,{=}\, s'_i$), the missingness probability is the same for all missingness indicators $r \,{\in}\, R$, formally:
    \mbox{$\pr(\Rfunc(\rand z) \!=\! r \!\mid\! \rand z\! \sim \! M(s)) \!=\!$ 
    $\pr(\Rfunc(\rand z') \!=\! r \!\mid\! \rand z' \!\sim\! M(s'))$.} 
    \item \textbf{MAR} iff for all $s, s' \in S$ and $z \in Z$, 
    if $z \preceq s, s'$, the probability of its indicator $r \coloneqq \Rfunc(z)$ is equal for $s$ and $s'$:
    \mbox{$\pr(\Rfunc(\rand{z}') \!=\! r \!\mid\! \rand{z}' \!\sim\! M(s))$} 
    {$\!=\! \pr(\Rfunc(\rand{z}'') \!=\! r \!\mid\! \rand{z}'' \!\sim\! M(s')).$}
    \end{itemize}
\end{definition}

\begin{example}\label{ex:MAR}
    We redefine $M$ in the \missMDP from \cref{ex:missing-pomdp} to be simple MAR:
    $M((s_1,\sta)) = \{ (s_1,\sta) \mapsto 1 \}$, and $M((s_1,\stb)) = \{ (s_1,\stb) \mapsto 0.5, (\mis,\stb) \mapsto 0.5 \}$.
    Here, the missingness probability of feature 1 depends on the \emph{always} observed value of feature 2.
    As an example of MAR, which is not simple MAR, consider:
    $M((s_1,\sta)) = \{ (s_1,\sta) \mapsto 0.5, (\bot,\bot) \mapsto 0.5 \}$, and $M((s_1,\stb)) = \{ (s_1,\stb) \mapsto 0.25, (\mis,\stb) \mapsto 0.25, (\bot,\bot) \mapsto 0.5 \}$.
    Here, feature 2 may go missing as well. 
    The missingness probability of feature 1 depends on the value of feature 2, only when it is \emph{observed}!
\end{example}

\begin{definition}[\textbf{MNAR}]
    The missingness function $M$ of a \missMDP~$\mdp$ is MNAR if and only if it is not MAR.
\end{definition}
For MNAR, missingness probabilities may depend on the values of missing features -- e.g., in~\cref{fig:missingness-MDP-visualization} the temperature feature influences its own missingness. 
In particular, in \emph{self-censoring} missingness functions, a feature's missingness probability depends on its own value.

\begin{example}\label{ex:MNAR}
    We adapt \cref{ex:missing-pomdp} to make $M$ MNAR and self-censoring for feature 2:
    $M((s_1,\sta)) = \{ (s_1,\sta) \mapsto 0.5, (s_1, \mis) \mapsto 0.5 \}$ and 
    $M((s_1,\stb)) = \{ (s_1,\stb) \mapsto 0.1, (s_1, \mis) \mapsto 0.9 \}$.
\end{example}

\begin{figure}[tb]
    \centering
        \includegraphics[width=0.95\linewidth]{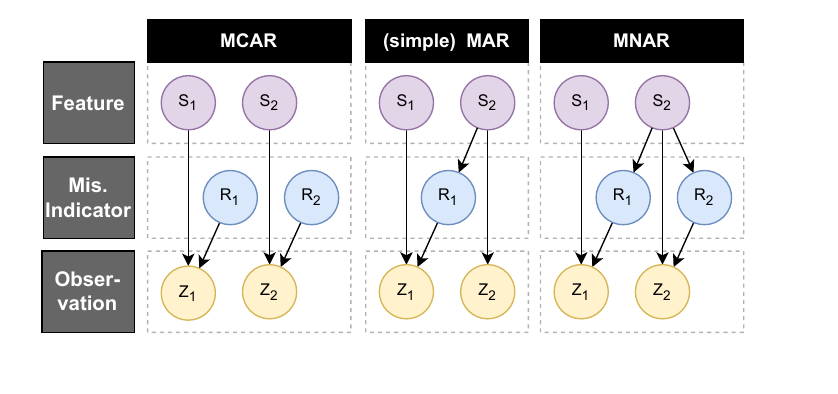}
    \vspace{-1em}
    \caption{Example missingness graphs visualizing relations between \missMDP elements for the three types of missingness functions.}
    \label{fig:missingness-graph-app-main}
\end{figure}

\subsection{Missingness graphs}
\label{sec:missingness-graphs}

\emph{Missingness graphs} (m-graphs) visualize the dependencies of missingness functions.
We adopt the definition from \citet{Mohan_graphical_2021} for our \missMDP framework.
An m-graph is a causal diagram \citep{pearl_causal_diagrams_1995} in the form of a directed acyclic graph. %
Its vertices correspond to variables, and the edges correspond to the relationships between the variables.

Vertices are grouped into three categories:
\mbox{\Snode-nodes} represent
features of the state space, 
\Znode-nodes represent the observation features, and
\Rnode-nodes represent the missingness indicators.%
\footnote{We exclude the category of unobserved features $U$ used in \citet{Mohan_graphical_2021}, as in our setting $U=\emptyset$, since $M$ depends on states.}
We omit the {\Rnode-nodes} for always observed features.
Arrows between nodes represent a causal relationship: The parent node is a direct cause of the child node. 
The absence of an edge intuitively indicates that two variables do not directly influence each other; formally, it means they are conditionally independent, given other variables in the graph, according to the d-separation criteria \citep{pearl_causality_2009}.

\paragraph{Visualizing types of missingness.}
\cref{fig:missingness-graph-app-main} uses m-graphs to illustrate the conditional independence assumptions of different types of missingness functions. %
For MCAR, both \Rnode-nodes are purely stochastic, having no incoming arrows and thus not depending on any feature value.
For (simple) MAR, there are two changes: 
Feature $S_2$ is always observable ($R_2$ is absent), and it affects missingness indicator $R_1$ (red arrow). 
For MNAR, $S_2$ can go missing, so $R_1$ depends on information that can go missing.
Note that m-graphs cannot represent \emph{context-specific} independence assumptions, which are needed to, e.g., represent non-simple MAR functions such as the one in \cref{ex:MAR}; but the missingness functions we focus on may all be represented by m-graphs.
We provide the corresponding m-graphs for all experiments in \cref{app:benchmarks}.

\section{Approximating missingness-MDPs}\label{sec:algos}
Our goal is to compute $\eps$-optimal policies for  \missMDPs with an unknown $M$.
For this, we first compute an approximation $\Mhat\approx M$ from the given dataset $\dataset$ of histories.
This yields an approximated, but fully specified \missMDP $\MDPhat$, which can be solved using any off-the-shelf POMDP solver.
Recall that in general, it's impossible to approximate unknown observation functions in POMDPs.
Thus, we identify \missMDPs where approximating $M$ with guarantees is possible.

\paragraph{Missingness types in focus.}
A necessary condition is that the missingness function can be approximated solely from observations, a property that missing data literature calls \emph{identifiability}~\citep{bhattacharya_identification_2020}.
Establishing new identifiability results is not the focus of this paper.
Instead, we provide PAC guarantees for types that are known to be identifiable.
Thus, we focus on missingness functions of type
(1)~MCAR,
(2)~simple MAR and 
(3)~non-self-censoring MNAR with independent missingness indicators.
Additionally, in \cref{sec:exp}, we experiment on non-identifiable MNAR.

\paragraph{Outline.}
\cref{rem:ignorability} presents an interesting insight orthogonal to our problem:
For maintaining a belief during policy execution, certain types of missingness can in fact be \emph{ignored}.
\cref{sec:algoritm-AMCAR-AsMAR,sec:algorithm-AIMI} describe our algorithms for approximating missingness functions.
Both are structured as follows: 
They state assumptions, define how to compute $\Mhat$, prove that the approximation is probably approximately correct, and explain how to utilize additional knowledge on the missingness function to reduce sample complexity.
\cref{sec:alg-policy} uses these algorithms to compute near-optimal policies.

\begin{remark}[Ignorability]\label{rem:ignorability}
Missing data literature defines \emph{ignorability} as cases where any quantity of interest can be consistently estimated from observations alone, and it is not necessary to model the missingness process~\citep{little_statistical_2019}. 
This holds under MCAR, and also under MAR whenever the quantity depends only on the observed features.  
We identify a similar notion of ignorability for \missMDPs: 
If the missingness function $M$ is MAR (including MCAR), then belief updates $\tau$ can be computed without knowledge of the precise probabilities of $M$, since these cancel out in Bayes' rule; see \cref{app:belief-oblivious} for a formal proof.
Thus, MAR missingness is ignorable for maintaining a belief when executing a policy in a \missMDP.
However, we stress that the missingness function \emph{is} required to compute belief-based policies, since the probabilities $\pr(b'\mid b,a)$ of successor beliefs $b'$ depend on it.
\end{remark}

\paragraph{Occurrence counts.}
Both algorithms extract the number of occurrences of every observation using the dataset $\dataset = (h_1,\dots,h_k)$ of 
$k$ histories $h_i \in \histories$.
For each \mbox{$h_i = \bigl(z^{(0)},\,a^{(0)},\,\,\dots, z^{(l)},\,a^{(l)}\bigr)$}, 
we denote the $j$-th observation $z^{(j)}$ by $h_i^{(j)}$. 
The number of occurrences of an observation $z \in Z$ is: 
$\#_{\dataset}(z) = \sum_{i=1}^k \sum_{j=0}^{\lvert h_i \rvert} \, \rand{1}_{h_i^{(j)} = z}$.
For a set $Z' \subseteq Z$, we define $\#_{\dataset}\bigl( Z'\bigr) = \sum_{z\in Z'} \#_\dataset(z)$.

\subsection{Approximating M for MCAR and simple MAR}
\label{sec:algoritm-AMCAR-AsMAR}

If a missingness function is of type simple MAR, we can approximate it using the \emph{approximation for simple MAR} algorithm, \AsMAR.
The modifications to obtain the algorithm for the more restricted MCAR-type functions, \AMCAR, are described at the end of the section.

\paragraph{Always-observable features.} 
Based on the dataset $\dataset$, we partition the feature indices $I$ into those that are always observed and those that can go missing:
\mbox{$\Ihatalw = \{ i \in I \mid \#_\dataset(\{z \in Z \mid z_i = \mis\}) = 0\}$}
and $\Ihatmis= I \setminus \Ihatalw$, respectively. 
Note, this partitioning is based on empirical data ($\Ihatalw\approx\Ialw$) and we might misclassify a feature index to be in $\Ihatalw$ even though it can go missing.

\paragraph{Computing $\Mhat$.}
We use the fact that $M$ can be seen as a mapping $S \rightarrow \Dist(R)$ (see paragraph  \enquote{Missingness indicators}, \cref{sec:missingness-indicator-definition}).
Consequently, for every state, we want to approximate the probability of a certain vector of missingness indicators.
The simple MAR assumption tells us that the probabilities can only depend on the features in $\Ihatalw$,
\rebuttal{meaning that $M(z \mid s) = M(z \mid s')$ when $s_i = s_i'$ for $i \in \Ialw$.}
Thus, for 
\rebuttal{all possible values} 
of the always-observable features of a state $s \in S$ and missingness indicator vector $r \in R$, we can compute the occurrence count $\#_\dataset(s, r) = \#_\dataset\left( Z_s^r \right),$ where $Z_s^r = \Big\{ z \in Z \mid \forall i \in I \colon $
$(i\in\Ihatalw \implies z_i=s_i)$ \textit{and} $
        (r_i=0 \! \implies z_i=\mis)
    \Big\}$.
Using $\#_\dataset(s, r)$, we obtain $\Mhat(z\mid s)$ as the fraction of observing $(s, \Rfunc(z))$ and the sum of counts for $s$ and all possible missingness indicators values:
\begin{equation}
    \label{eq:Mhat-AsMAR}
    \Mhat(z \mid s) 
    = \frac{\#_\dataset(s,\Rfunc(z))}{\sum_{r \in R} \#_\dataset(s, r)} .
\end{equation}

\paragraph{Probably approximately correct.}
With enough data, our approach yields an arbitrarily precise approximation of the true missingness function. 
We formalize this in \cref{thm:alg-1} as a PAC guarantee, not only proving that it becomes $\eps$-precise for every $\eps>0$, but that we can also bound the probability of an error (through unlucky sampling).
Additionally, we can adapt the claim to bound the imprecision of the resulting $\Mhat$ for a given dataset.
The proof is provided in \cref{app:algo-alg1}.

\begin{restatable}[PAC guarantee for \AsMAR]{theorem}{algOne}\label{thm:alg-1}
    Let $\mdp$ be a missingness-MDP with a simple MAR missingness function.
    For every precision $\eps$ and confidence threshold~$\delta$, there exists a number $n^*$, such that every dataset $\dataset$ of $n^*$ histories satisfies the following:
    With probability at least $\delta$, for $\Mhat$ computed on $\dataset$ according to \cref{eq:Mhat-AsMAR}, for all reachable states $s\in S$ and observations $z\in Z$, we have \mbox{$|\Mhat(z \mid s) - M(z \mid s)| \leq \eps$}.    
    Dually, for a dataset $\dataset$ and confidence threshold $\delta$, there exists an~$\eps$ such that with probability at least $\delta$, for all reachable states $s\in S$ and observations $z\in Z$, we have the same inequality, i.e.\ $|\Mhat(z \mid s) - M(z \mid s)| \leq \eps$.
\end{restatable}

\paragraph{Using additional assumptions on the missingness function.}
Beyond the necessary simple MAR assumption, we can exploit additional assumptions to improve the approximation of $M$ for the same~$\dataset$.
Consider a feature $i$ that is always observable, but does not affect the missingness probability of other features. 
We can exclude such $i$ from $\Ihatalw$, combining the occurrence counts of states that differ only in this feature.
Hence, if we assume $M$ to be {MCAR},
$\Ihatalw$ reduces to an empty set.
Consequently, we get that $\#_\dataset(s,r)$ is independent of $s$, and we effectively only count occurrences of missingness indicators, 
resulting in the algorithm \AMCAR.
We prove the correctness of these improvements in \cref{app:algo-alg1}.
In \cref{sec:exp}, we empirically show the improved precision of $\Mhat$ estimated from the same $\dataset$ with this approach.

\subsection{Approximating M with independent missingness indicators }\label{sec:algorithm-AIMI}

This section presents the \textit{approximation for independent missingness indicators} algorithm, \AIMI. 
It may be applied to $M$ of types MCAR, simple MAR as well as a subset of identifiable MNAR which satisfy the following requirements: 
\begin{enumerate}[topsep=2pt, itemsep=1.5pt, parsep=1.5pt, partopsep=0.5pt, leftmargin=20pt, labelsep=0.5em, labelindent=0pt, labelwidth=!]
    \item \textbf{Independence of missingness indicators:} 
        The fact that one feature is missing must not influence the missingness-probability of any other feature. %
        Formally, for $s \in S$ and $z \in Z$, 
        \mbox{$\pr(\rand z \mid \rand z \sim M(s)) = \Pi_{i\in I} \pr(\rand z_i \mid \rand z \sim M(s))$.}
    \item \textbf{No self-censoring:}
        Intuitively, a feature may not influence its own missingness probabilities. 
        Formally, for all $i\in I$ and every pair of states \mbox{$s,s'\in S$} that differ only in the $i$-th feature ($s_i \neq s'_i$, but for all $j\neq i$ we have $s_j=s'_j$) we have 
        \mbox{$\pr(\rand z_i = \bot \mid \rand z \sim M(s)) = \pr(\rand z_i = \bot \mid \rand z \sim M(s'))$.}
    \item \textbf{Positivity:} 
        Intuitively, if a feature affects the missingness probabilities of other features, we need to observe its value to learn the missingness probabilities.
        However, this is impossible if it always misses.
        Therefore, we require a \emph{positivity assumption}~\citep{hernan_causal_2020}:
        For all $i\in I$ and $s \in S$, we have $\pr(\rand z_i \neq \bot \mid \rand z \sim M(s)) > 0$.  %

\end{enumerate}

\paragraph{Computing $\Mhat$.}
We compute the occurrence count for every state $s\in S$, feature $i\in I$ and value of a corresponding $i$-th missingness indicator $r_i\in \{0,1\}$ as
$
    \#_\dataset(s, i, r_i) = \#_\dataset(Z^{i,r_i}_s),
$
where $Z^{i,r_i}_s = \Big\{ z \in Z \mid \forall j \in I\setminus\{i\} \colon 
     (z_j=s_j)$ \textit{and} $(r_i=0 \iff z_i = \bot)
\Big\}$.
By positivity, a large enough dataset almost surely contains observations to make the counters non-zero (i.e.\ for all $s$ and $i$, we have $\#(s,i,0)+\#(s,i,1)>0$).
The probability of a non self-censoring feature $i$ depends only on the other features $j\in I\setminus\{i\}$.
Finally, using the independence assumption, we can infer $\Mhat$ by taking the product of the individual missingness probabilities of all features. %
\begin{equation}\label{eq:Mhat-AIMI}
    \Mhat(z \mid s) = 
    \prod_{i\in I} \frac{\#_\dataset(s, i, \Rfunc(z)_i)}{\#_\dataset(s, i, 0) + \#_\dataset(s, i, 1)}.
\end{equation}

\paragraph{Probably approximately correct.}
In \cref{app:algo-alg2}, we prove \cref{thm:alg-2} that provides the same kind of guarantee as in \cref{thm:alg-1}; the only difference are the assumptions on the missingness function and the approach for calculating $\Mhat$.

\begin{figure*}[tbh]

  \captionsetup[subfigure]{labelformat=empty} %
  \captionsetup[subfigure]{textformat=empty}  %

    \begin{minipage}{\textwidth}
        \includegraphics[width=0.85\textwidth]{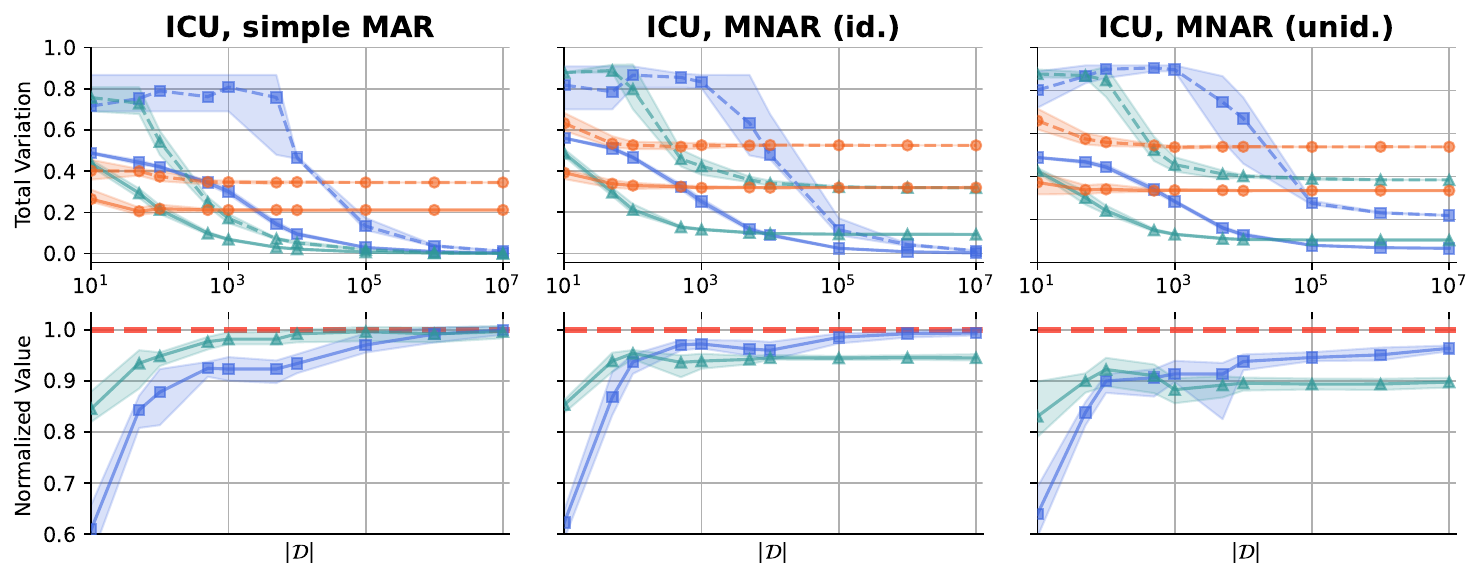}
        \includegraphics[width=0.14\textwidth, trim= 0 -50 0 0, clip]{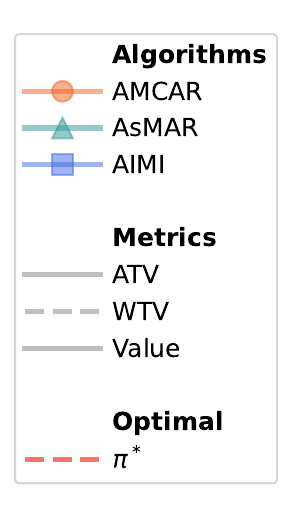}
    \end{minipage}\\
    \includegraphics[width=0.58\textwidth]{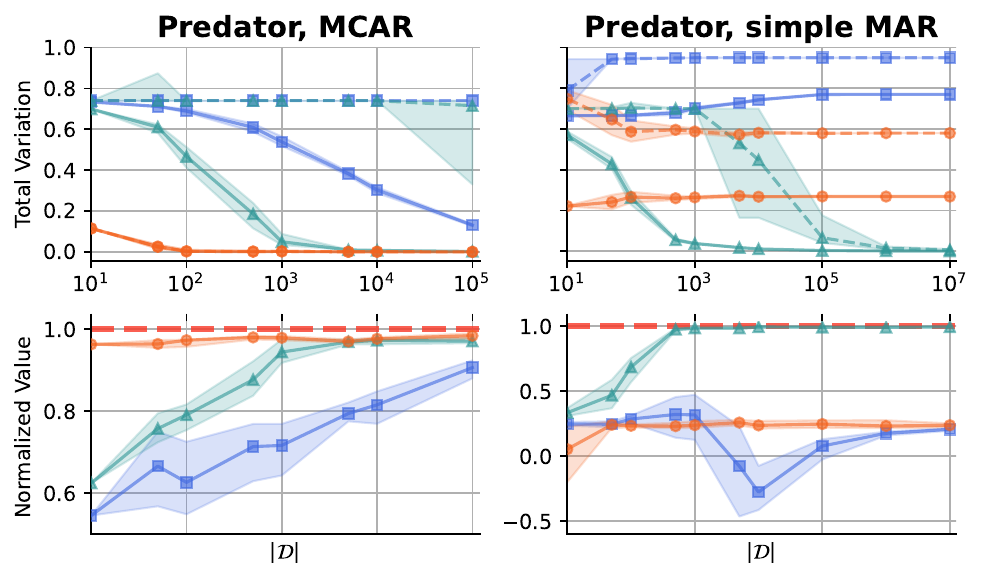}
    \topleftframe[1pt]{gray!40}{\includegraphics[width=0.4\textwidth, trim= 0 -20 0 -10, clip]{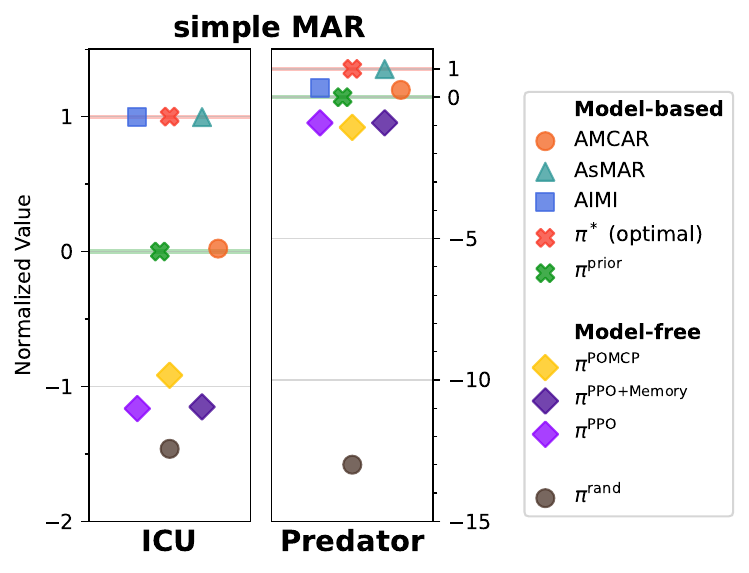}}
    \begin{subfigure}{0.001\textwidth}
        \caption{}
        \label{fig:benchmark-results}
    \end{subfigure}
    \begin{subfigure}{0.001\textwidth}
        \caption{}
        \label{fig:baseline-comparison}
    \end{subfigure}
    \caption{\textbf{Left} (and top)\textbf{:} Empirical results for the \ICU and \Predator benchmarks, with average/worst TV~(ATV/WTV) of $\Mhat$ and resulting policy values $V_{\mdp}(\pi)$. 
    Values are normalized so that 1 corresponds to the \emph{optimal} policy \piopt using the true $M$ and 0 to the \emph{prior} policy \piguess.
    \textbf{Bottom right:} Comparison of normalized values on \ICUSimpleMAR and \PredatorSimpleMAR between model-based and the model-free baseline methods.
    }
    \label{fig:results-figure}
\end{figure*}

\begin{restatable}
[PAC guarantee for \AIMI]
{theorem}{algTwo}\label{thm:alg-2}
    Let $\mdp$ be a \missMDP where the missingness function satisfies independence, non-self-censoring, and positivity.
    Then, $\Mhat$ computed using \cref{eq:Mhat-AIMI} offers the same PAC guarantees as in \cref{thm:alg-1}.
\end{restatable}

\paragraph{Using additional assumptions on the missingness function.}
In general, \AIMI maintains a counter for every combination of the feature valuations of other features $j \in I \setminus \{i\}$.
If we know that a certain feature $j$ does not affect the missingness probability of $i$ -- there is no edge between the $j$-th \Snode-node and the $i$-th \Rnode-node -- we merge the counters for all values of the $j$-th feature.
This knowledge comes from 
\textbf{(a)} an m-graph, 
\textbf{(b)} assuming simple MAR while observing feature $j$ goes missing in $\dataset$, or 
\textbf{(c)} assuming MCAR, in which case we drop the dependency on $s$ in the counters.
We prove in \cref{app:algo-alg2} that all these modifications retain the PAC guarantees.

\subsection{Computing a policy with the approximations}\label{sec:alg-policy}
Here, we provide the final step of our problem statement: computing a PAC policy from the approximated \missMDP. 
We show in \cref{app:algo-policies} that after finitely many samples, $\Mhat$ is accurate enough to yield an $\eps$-optimal policy.
We highlight that learning $\Mhat$ to precision $\eps$ is insufficient, as the errors in $\Mhat$ aggregate when solving the \missMDP. 

\begin{restatable}[Computing $\eps$-optimal Policies]{theorem}{policies}\label{thm:policies}
    Let $\mdp$ be a \missMDP with a missingness function that is simple MAR or that satisfies independence, no self-censoring, and positivity.
    Assume we can sample histories collected under a fair policy, and we know a lower bound on the smallest missingness probability $p\leq \min_{s\in S,z\in Z} M(z\mid s)$.
    Then, for every given precision $\eps$ and confidence threshold $\delta$, we can in finite time compute a policy $\pi^*$ such that with probability at least $\delta$ it is $\eps$-optimal, i.e.\ $(\sup_{\pi} V_{\mdp}(\pi)) - V_{{\mdp}}(\pi^*) \leq \eps$.
\end{restatable}

\paragraph{Practical considerations.}
\Cref{thm:policies} concern asymptotic convergence to an $\eps$-optimal policy, providing the theoretical foundation of our approach.
In practice, the required number of samples is very large, and we work with datasets that are not necessarily sufficient to provide the $\eps$-optimality guarantees.
Datasets of limited size may cause a practical problem:
For an observation $z$ with \mbox{$\#_\dataset(s,\Rfunc(z)) = 0$}, 
for any $s \in S$ we obtain $\Mhat(z \mid s) = 0$,
leading to a division by zero for $s$ when performing the belief update $\tau$.
We circumvent this case by setting  
\mbox{$\#_{\dataset, \smallextra}(s,r) = \#_\dataset(s,r) + \smallextra$}, i.e.\ we add a small $\smallextra > 0$ to every count. 
The influence of $\smallextra$ diminishes with an increasing dataset size $|\dataset|$.

\section{Experiments}
\label{sec:exp}
We conduct an experimental evaluation to test our theoretical results and to compare our methods against several baselines.
We aim to answer the following questions:

\begin{questionenum}[topsep=2pt, itemsep=1.5pt, parsep=1.5pt, partopsep=0.5pt, leftmargin=20pt, labelsep=0.5em, labelindent=0pt, labelwidth=!]
    \item \label{q:adequate_approximation}
    Do our proposed learning algorithms adequately approximate the missingness function?
    \item \label{q:missingness_approximation}
    How do (in)correct assumptions on the missingness function affect the approximation?
    \item \label{q:value_convergence}
    With increasing dataset size, does the value of the policy computed on the approximated \missMDP converge to that of the optimal policy on the true \missMDP? 
    \item \label{q:baselines}
    How does the value computed from the approximated \missMDP compare against model-free baselines?
\end{questionenum}
Below, we first describe the experimental setup.
Then, we analyze the results in the order of the above questions.

\paragraph{Benchmarks.}
We consider two environments with varying missingness types:
\ICU models a doctor treating a patient, whose vital measurements are not always available~\citep{johnson_alistair_mimic-iv_2022}%
, and
\Predator, a variant of the Tag environment~\citep{DBLP:conf/ijcai/PineauGT03}, where a predator chases a partially hidden prey.
For our benchmarks, we consider the following four missingness types:
{(1)} \MCAR,
{(2)}~\simpleMAR, a simple MAR function,
{(3)} \MNARid, an identifiable MNAR function (fulfilling the required assumptions of \AIMI), and
{(4)} \MNARunid, an unidentifiable MNAR function with self-censoring.
For \Predator and all missingness functions, the $(x,y)$-coordinates of the prey always go missing jointly, i.e.~the missingness indicators are dependent;
for \ICU, the missingness indicators are independent.
As mentioned in \cref{app:benchmarks}, \ICU has 800 reachable states and \Predator up to 1311.

\paragraph{Setup.}
For a range of dataset sizes $|\dataset|$, we collect data using the uniform random policy \pirand with 
\mbox{$\forall a \in A, \pirand(a\mid \cdot) = \nicefrac{1}{|A|}$}, 
and compute the estimate $\Mhat \approx M$ using our proposed algorithms: 
\AMCAR, \AsMAR, and \AIMI \mbox{(\cref{sec:algos})}.
Each $\Mhat$ yields an approximated \missMDP $\MDPhat$, for which we compute a policy \pihat using the POMDP solver SARSOP~\citep{kurniawati2008sarsop}.
We consider the following baselines: 
\textbf{(1)}~\emph{optimal}: the SARSOP policy \piopt computed for the true $M$ (the upper bound);
\textbf{(2)}~\emph{prior}: the SARSOP policy \piguess with an uninformed prior of $M$, where each feature independently goes missing with probability $0.5$;
\textbf{(3)} \emph{POMCP}~\citep{silver2010monte}: generating an online POMDP planning policy \piPOMCP;
\textbf{(4)} \emph{PPO}~\citep{DBLP:journals/corr/SchulmanWDRK17}: the policies \piPPOmemory and \piPPO, with and without memory.
We group \piPPO and \piPOMCP as \emph{model-free} baselines as they rely solely on simulation access to the \missMDP.
We provide more extensive details of the setup in \cref{app:benchmarks}.
Moreover, we publicly provide the code of our implementation\footnote{\url{https://github.com/ai-fm/missingness-pomdps}}.
For further details on reproducibility, see \cref{app:reproducibility}.

\paragraph{Metrics.}
For every dataset size and method, we perform $20$ independent runs and report the average together with the interquartile range~(shaded area) of the following metrics:

\vspace{-0.65\parskip}
\begin{enumerate}[topsep=2pt, itemsep=1.5pt, parsep=1.5pt, partopsep=0.5pt, leftmargin=20pt, labelsep=0.5em, labelindent=0pt, labelwidth=!]
\item To empirically assess the quality of the approximation $\Mhat$ compared to the true $M$, we compute the \textbf{total variation}~(TV) of the distributions at a state $s \in S$ as %
$TV(s) = \frac{1}{2} \sum_{z\preceq s} \left|\Mhat(z\mid s) - M(z\mid s)\right|$.
We aggregate the TV across states by the average TV~(ATV): $\nicefrac{1}{|S|}\sum_s TV(s)$, and the worst TV~(WTV): $\max_s TV(s)$.
\item %
We asses how the various \pihat from the algorithms perform on the true \missMDP $\mdp$ by comparing their \textbf{value}~$V_{\mdp}(\pihat)$ to $V_{\mdp}(\piguess)$ and the optimum $V_{\mdp}(\piopt)$.
All policy values are normalized s.t.~$1$ and $0$ correspond to the values of the \emph{optimum} and \emph{prior} baselines, respectively.
\end{enumerate}

\paragraph{Results.}
The experimental results in Figure~\hyperref[fig:results-figure]{\labelcref{fig:results-figure} (left, top)} present how the TV of the approximated missingness function $\Mhat$ and the value of the associated policy \pihat evolve with dataset size $|\dataset|$.
\cref{fig:baseline-comparison} compares our algorithms to several baselines for selected benchmarks.
We report the best baseline performance across multiple seeds.
We provide baseline results for all benchmarks as well as results for an additional \PredatorMNARunid benchmark in \cref{app:benchmarks}.

\paragraph{
\cref{q:adequate_approximation}:
    With sufficient data, our algorithms adequately approximate the missingness function.
}
Under the appropriate missingness assumptions, each algorithm learns the missingness function (bringing the TV near zero): \AMCAR learns the exact missingness function in \PredatorMCAR within 100 observations.
We observe similar results for \AsMAR in \ICUSimpleMAR and \PredatorSimpleMAR, and for \AIMI in \ICUMNARid. 
\vspace{2ex}

\paragraph{
\cref{q:missingness_approximation}:
    The assumptions on the missingness function significantly affect the quality of the approximation. 
}
Under correct assumptions on the missingness function, the TVs of the approximated $\Mhat$ approach zero (e.g. \AIMI on \ICUSimpleMAR and \ICUMNARid).
Further, relaxations on the missingness (e.g. assuming MAR despite MCAR missingness) allows the approximation of $M$, but may lead to significant data inefficiency (see \AMCAR vs. \AsMAR on \PredatorMCAR).
Breaking the missingness assumptions results in converging to a poor WTV (see \AIMI on all Predator benchmarks).

\paragraph{
\cref{q:value_convergence}:
    Convergence to the optimal policy follows the quality of the approximation.
}
With a sufficiently accurate approximation of the missingness function, the values of our method's policies converge to the optimal policy's \piopt values (e.g. \AIMI on \ICUSimpleMAR and \ICUMNARid, \AMCAR and \AsMAR on \PredatorMCAR, \AsMAR on \PredatorSimpleMAR). 
Despite the missingness of \ICUMNARunid not adhering to the necessary assumptions, \AIMI approximates the relevant parts of $M$ well enough to yield a policy that is considerably close to \piopt.

\paragraph{
\cref{q:baselines}: 
Baseline methods, including PPO and POMCP, are not competitive.
}
\cref{fig:baseline-comparison} compares our algorithms to several baselines.
Our algorithms, under correct missingness assumptions (see \AsMAR or \AIMI for \ICU, or \AsMAR for \Predator), consistently outperform policies derived from an uninformed prior over the missingness function (\piguess).
Meanwhile, all model-free approaches---\piPPO, \piPPOmemory, and \piPOMCP---significantly underperformed on both benchmarks. 
We attribute this gap to two factors. 
First, model-free methods solely access the transition function through simulation. 
Second, in a separate experiment without missingness, these methods successfully learned the optimal policy for the fully observable MDP. 
This indicates that the poor performance stems from these methods' inability to effectively handle the specific type of partial observability of \missMDPs.
We attribute the better performance (compared to the random policy \pirand) of the model-free approaches in \Predator, in contrast to \ICU, to the lower stochasticity of the underlying transition function of \Predator.

\section{Conclusion}
\MissMDPs integrate the theory of missing data with decision-making under uncertainty.
From data generated by a \missMDP with known underlying MDP, we approximate the unknown missingness function, which -- under assumptions about its type -- enables the computation of $\varepsilon$-optimal policies.
We show that incorrect assumptions about the missingness type may lead to misspecified models and suboptimal policies.
Our experiments support the theory and demonstrate the practical benefits of our approach, highlighting the superiority over model-free baselines in our problem setting. %
Future work may lift the assumption of a known transition function and extend \missMDPs to the more general setting of miss-POMDPs.

\section*{Acknowledgements}
This work was supported by the European Research Council (ERC) Starting Grant 101077178 (DEUCE), and the Czech Science Foundation grant GA23-06963S (VESCAA).

\bibliography{references_tidy}
\bibliographystyle{ijcai26named}

\appendix

\section{Reproducibility}\label{app:reproducibility}

All code required for conducting experiments is included in an anonymized repository at \url{https://anonymous.4open.science/r/missingness-pomdps}.
The range of values tried per hyperparameter during development of the paper are shown in \cref{tab:parameter-search}, and the final values were chosen based on best cumulative reward obtained.
The final hyperparameters used for each model and algorithm are provided in the repository (see our \texttt{README.md}).
The computing infrastructure in which the experiments have been carried out is described in \cref{app:setup}.

\section[Proofs for ignorability]{Proofs for \cref{sec:missingness}: ignorability}
\label{app:proofs}
\label{app:belief-oblivious}

\begin{lemma}\label{thm:ignorability}
    If a missingness function $M$ is MAR, then
    \[
    \forall z \in Z, \exists p \in [0, 1], \forall s \in S, M(z \mid s) = \indicatorfunction_{z \preceq s} \cdot p.
    \]
\end{lemma}
\begin{proof}
Suppose that $M$ is MAR.
The lemma states that $\forall z \in Z$, $\exists p \in [0, 1]$, $\forall s \in S$, $M(z \mid s) = p$ if $z \preceq s$ and otherwise $M(z \mid s) = 0$.
Since $z \not \preceq s$ implies that $M(z \mid s) = 0$, we only need to show that
$\forall z \in Z, \exists p \in [0, 1], \forall s \in S, z \preceq s \Rightarrow M(z | s) = p$,
which directly follows from the MAR assumption.

\end{proof}

\begin{remark}
\cref{thm:ignorability} implies that the missingness function can be omitted in the belief update.
Let $b \in \beliefset$ be a belief, and let $s' \in S$. Then, for any $a \in A$ and $z \in Z$, it holds that
\begin{align*}
    b'(s') &= \tau(b, a, z)(s')  \\
     &\coloneqq\frac{M(z \mid s') \sum_{s\in S} \T(s'\mid s,a)b(s)}{\sum_{s'' \in S} M(z \mid s'') \sum_{s\in S} \T(s''\mid s, a)b(s)} 
     \tag{By definition of belief update}\\
    &= \frac{\indicatorfunction_{z \preceq s'} \cdot p \sum_{s\in S} \T(s'\mid s,a)b(s)}{\sum_{s'' \in S} \indicatorfunction_{z \preceq s''} \cdot p \sum_{s\in S} \T(s''\mid s, a)b(s)} 
    \tag{By \cref{thm:ignorability}}\\
    &=\frac{ \indicatorfunction_{z \preceq s'} \sum_{s\in S} \T(s'\mid s,a)b(s)}{\sum_{s'' \in S} \indicatorfunction_{z \preceq s''} \sum_{s\in S} \T(s''\mid s, a)b(s)}.
    \tag{$p$ cancels out}
\end{align*}
Therefore, the probabilities of $M$ do not affect the resulting probabilities of the belief update.
In particular, this means that maintaining a belief while executing a \missMDP does not require knowledge of $M$.

Still, we stress again that one needs $M$ to compute an optimal policy because this requires constructing and solving the belief MDP (see \cite[Chapter 16.4.1]{AIMA}), which in turn requires knowing the probability $\pr(b'\mid b,a)$ of going to a successor belief $b'$ from a current belief $b\in\beliefset$ upon playing action $a\in A$.
Concretely, the probability of a successor belief $b'=\tau(b,a,z)$ depends on the probability of $z \in Z$ given $b$ and $a$, which in turn depends on $M$,
\begin{align*}
\pr(b'\mid b,a) &= \sum_{z \in Z} \pr(z\mid b,a) \indicatorfunction_{b' = \tau(b,a,z)},\\
\pr(z\mid b,a) &= \sum_{s\in S} b(s)\sum_{s' \in S} \T(s'\mid s,a) M(z \mid s').
\end{align*}
Here, no normalization occurs, and the probabilities of $M$ do not cancel out.
\end{remark}

\section[Proofs for Probably Approximately Correct]{Proofs for \cref{sec:algos}: probably approximately correct} \label{app:algos}

This appendix is about proving that given enough data, we can approximate the missingness function to arbitrary precision $\eps$, or the other way round: we can prove a certain precision $\eps$ for any given dataset $\dataset$.
In both directions, we provide a probabilistic guarantee, i.e.\ that the result is correct with probability at least $\delta$.
The reason the guarantee has to be probabilistic is that our knowledge relies on a sampled dataset, and, intuitively, there always is a chance that we were \enquote{unlucky} and received a very unlikely sequence of samples from which we infer a wrong approximation.

\paragraph{Outline.}
First, in \cref{app:algo-bernoulli} we recall standard notions from statistics literature: Bernoulli processes and the fact that building on Okamoto's inequality, we can obtain a size for our dataset $\dataset$ given precision $\eps$ and confidence $\delta$ (or, analogously, obtain a precision $\eps$ given $\dataset$ and $\delta$).
Afterwards, \cref{app:algo-alg1} and \cref{app:algo-alg2} provide the proofs of \cref{thm:alg-1,thm:alg-2}, respectively, i.e.\ the guarantees for our algorithms. Moreover, they prove the guarantees for the modified algorithms when using more information about the missingness function.
Finally, \cref{app:algo-policies} proves \cref{thm:policies}, our main result that $\eps$-policies can be computed.

\subsection{Bernoulli processes}\label{app:algo-bernoulli}

\begin{definition}[Bernoulli process~\cite{bernoulli1713ars}, {\cite[Chapter 4.3]{dekking2005modern}}]
    A Bernoulli process is a sequence of binary random variables that are independent and identically distributed.
    All random variables have probability $p$ to yield a 1, and probability $1-p$ to yield a 0.
\end{definition}

Throughout this appendix, we write $n$ for the length of the sequence of a Bernoulli process, and $k$ for the number of successes, i.e.\ the number of times it yielded a 1.
Moreover, we denote by $\hat{p} = \frac k n$ the empirical success probability.
Okamoto's seminal work proves the following property of estimating $p$ through observing a Bernoulli process:

\begin{theorem}[Okamoto's inequality~{\cite[Theorem 1]{Okamoto59}}]
    For a Bernoulli process with $n$ repetitions and $k$ successes and a given precision $\eps$, we have
    \[
        \Pr(\hat{p} - p \geq \eps) \leq \mathsf{e}^{-2 \cdot n \cdot \eps^2} \text{ and }
        \Pr(p - \hat{p} \geq \eps) \leq \mathsf{e}^{-2 \cdot n \cdot \eps^2}.
    \]
\end{theorem}
Combining these, we get that $\Pr(\abs{\hat{p} - p} \geq \eps) \leq 2 \cdot \mathsf{e}^{-2 \cdot n \cdot \eps^2}$, in words: The probability of the estimate~$\hat p$ being more than $\eps$ away from the true probability $p$ is less than $2 \cdot \mathsf{e}^{-2 \cdot n \cdot \eps^2}$.
For our guarantees, we want to be $\varepsilon$-precise 
with probability at least $\delta$, so the probability of error should be upper bounded by $1-\delta$.%
\footnote{Note that in this paper, we use $\delta$ as the probability of the estimate being correct, unlike e.g.\ \cite{WitS}, where $\delta$ is the probability of an error.} 
Thus, we require $2\cdot \mathsf{e}^{-2\cdot n \cdot \eps^2} \leq {1-\delta}$.
Then, we can solve the inequality for $\eps$ or $n$:
\begin{equation}
    \label{eq:okamoto}
    2\cdot \mathsf{e}^{-2\cdot n \cdot \eps^2} \leq {1-\delta} \, \Leftrightarrow \, \eps \geq \sqrt{\frac{\ln(\frac{2}{1-\delta})}{2\cdot n}} \, \Leftrightarrow \, n \geq \frac{\ln(\frac{2}{1-\delta})}{2\cdot \eps^2}.
\end{equation}

In other words, given two of precision $\eps$, confidence $\delta$, and number of repetitions $n$, we can infer the third.
We remark that there exist other inequalities similar to Okamoto's that yield the same result, but with tighter bounds; we refer to~\cite[Section 3]{WitS} for a discussion.
However, as our goal is only to prove the existence of a bound, we choose the conservative Okamoto bound for its easier accessibility.

\subsection{PAC guarantees for AsMAR}\label{app:algo-alg1}

\algOne*

\begin{proof}
    \textbf{Proof outline.}
    We first show that the computation of every $\Mhat(z\mid s)$ is related to a Bernoulli process. 
    Then, using the results of \cref{app:algo-bernoulli}, we can prove the claims of the theorem for individual state-observation pairs.
    Next, we lift this to all state-observation pairs by distributing the confidence $\delta$.
    Finally, we individually explain how this yields the two claims of the theorem.

    \medskip
    \textbf{The Bernoulli process related to $\Mhat(z\mid s)$.}
    Fix a state $s\in S$ and an observation $z \in Z$.
    Consider the following random variable:
    Sample a state $s'\in S$ and the corresponding observation $z'\in Z$.
    Set the random variable to 1 if $\forall i \in I \colon 
        (i\in\Ialw \implies z'_i=s_i) \wedge 
        (\Rfunc(z)_i=0 \implies z'_i=\mis )$; set the random variable to 0 if $\forall i \in I \colon (i\in\Ialw \implies z'_i=s_i)$; and ignore the sampled $(s',z')$ otherwise, i.e.\ if $\exists i\in I \colon (i \in \Ialw \wedge z'_i \neq s_i).$
    Note that the random variable is 1 exactly when the sample would be counted by $\#_\dataset(s, \Rfunc(z))$, and the sample is not ignored exactly when it would be counted by $\sum_{r \in R} \#_\dataset(s, r)$.

    We require that the probability of the random variable being 1 is equal among all sampled state-observation pairs $(s',z')$ that are not ignored by it, and moreover we require this probability to be equal to $M(z\mid s) = M(\Rfunc(z) \mid s) \eqqcolon p$.
    To prove this, we use the assumption that $M$ is a simple MAR missingness function; thus, we know that for all $s'$ that agree with $s$ on all always observable features (formally: $\forall i \in I \colon (i\in\Ialw \implies z'_i=s_i)$) , we have $p = M(\Rfunc(z) \mid s) = M(\Rfunc(z) \mid s')$.

    We have just shown that the random variable we constructed is a Bernoulli process with success probability $p = M(z\mid s)$, with the number of repetitions $n = \sum_{r \in R} \#_\dataset(s, r)$ and the number of successes $k=\#_\dataset(s, \Rfunc(z'))$.
    Note that the definition of $\Mhat$ in \cref{eq:Mhat-AsMAR} is exactly the empirical success probability $\hat p = \frac k n$.

    Observe that we do not need a separate Bernoulli process for every state-observation pair:
    The number of repetitions $\sum_{r \in R} \#_\dataset(s, r)$ is independent of the observation $z$, since that only affects whether it is counted as success or not.
    Further, it suffices to have one random variable per combination of valuation for the features in $\Ialw$, since all states that agree on the always observable features yield the same Bernoulli process.
    Moreover, we do not need to consider every observation $z$ (as this includes observations that do not admit $s$), but rather only every missingness indicator vector $r \in R$.
    In the following, we still write \enquote{Every state-observation pair} instead of \enquote{Every pair of set of states that agree on the always observable features and missingness indicator vector}, as it is also true and more concise.

    \medskip
    \textbf{Single state-observation pair.}
    Consider the Bernoulli process just described for a fixed state-observation pair $(s,z)$.
    We explain how to use the results of \cref{app:algo-bernoulli} towards proving the first and second claim of the theorem:
    \begin{itemize}
        \item First claim: By the third variant of \cref{eq:okamoto}, we have that given a precision $\eps$ and confidence threshold $\delta_{s,z}$, we can compute a necessary number of samples $n_{s,z}$ such that we obtain the PAC guarantee for this state-observation pair. 
        \item Second claim: Observe that a given dataset $\dataset$ corresponds to a number of repetitions of every Bernoulli process. Let $n_{s,z}$ be the number of repetitions for the pair $(s,z)$. Thus, using the second variant of \cref{eq:okamoto}, we have that given $\dataset$ (and thus $n_{s,z}$) and a confidence threshold $\delta_{s,z}$, we can compute a precision $\eps_{s,z}$ such that we obtain the PAC guarantee for this state-observation pair.
    \end{itemize}

    \medskip
    \textbf{All state-observation pairs.}    
    We can split the given confidence threshold $\delta$ uniformly over all state-observation pairs, i.e.\ for every $s\in S$, $z\in Z$, we have $\delta_{s,z} = \frac{\delta}{\abs{S} \cdot \abs{Z}}$.
    Then, by the union bound, the probability of all state-observation pairs being correctly estimated is the sum of all $\delta_{s,z}$, which (since we distributed it uniformly) is $\delta$.
    By splitting the confidence threshold in this way, we can obtain the PAC guarantee for all state-observation pairs.
    
    \medskip
    \textbf{Second claim.}
    We first provide the full argument for the second claim, as it is simpler.
    Given the dataset $\dataset$ and confidence threshold $\delta$, we obtain an $\varepsilon_{s,z}$ for all state-observation pairs.
    The probability that all of these are correct is at least $\delta$.
    We obtain the claim by taking the maximum over these, i.e.\ setting $\eps \coloneqq \max_{s\in S, z\in Z} \varepsilon_{s,z}$.
    Then we have that with probability at least $\delta$, for all states $s\in S$ and observations $z\in Z$, we have $\abs{\Mhat(z \mid s) - M(z \mid s)} \leq \eps$.

    \medskip
    \textbf{First claim.}
    We proceed in two steps: 
    We explain the analogous argument to the second claim, based on an assumption on the dataset. 
    Afterwards, we explain how this assumption on the dataset can be satisfied.

    Assume that for every state-observation pair $(s,z)$, the dataset $\dataset$ contains at least $n_{s,z}$ samples, i.e.\ the number computed using \cref{eq:okamoto} inserting $\eps$ and $\delta_{s,z}$.
    Then, analogously to the proof of the second claim, computing $\Mhat$ using this dataset satisfies that with probability at least $\delta$, for all states $s\in S$ and observations $z\in Z$, we have $\abs{\Mhat(z \mid s) - M(z \mid s)} \leq \eps$.

    It remains to show that there exists a number $n^*$ such that a sampled dataset of $n^*$ histories has the required property.
    For this, we have to spend some of our confidence threshold $\delta$, since we can only guarantee the property with a certain probability; there is the chance that even upon sampling $n^*$ histories, we are unlucky and some state-observation pair has not been sampled often enough.
    Thus, we split $\delta$ as follows: $\delta_{\dataset}$ is used to guarantee the property of the dataset, and $\delta_{\Mhat}$ is used to guarantee the consequential property of $\Mhat$. 
    Thus, $\delta_{s,z}$ above are obtained by uniformly distributing $\delta_{\Mhat}$, not all of $\delta$.
    Then, by the union bound, the probability that $\dataset$ has the desired property and that the PAC guarantee holds is $\delta_{\dataset} + \delta_{\Mhat} = \delta$.

    We now need to show that there exists an $n^*$ such that a dataset of this size contains the required number of samples with probability at least $\delta_{\dataset}$.    
    Recall that the dataset is sampled using a fair policy, which means that every state has a positive probability to be visited; thus (assuming that the length of every history is at least as large as the number of states in the \missMDP), there exists a minimum probability $m$ such that every state is visited with at least probability $m$ in every history.
    Moreover, observe that for a state-observation pair $(s,z)$, the number of samples for its Bernoulli process is at least the number of times $s$ has been visited; this is because a sample is used when it agrees with $s$ on the always observable features. 
    Thus, for every sampled history, we have a probability of at least $m$ to obtain at least one sample for $(s,z)$.
    This lower bound on the number of samples for $(s,z)$ is binomially distributed with success probability $m$ \cite[Chapter 4.3]{dekking2005modern}.
    Thus, there exists a number of histories $n^*$ such that the probability of having at least $n_{s,z}$ samples for $(s,z)$ when sampling at least $n_h$ histories is greater than $\delta_{\dataset}$.
    As before, this argument was for a single state-observation pair; thus, $\delta_{\dataset}$ is also uniformly distributed over all state-observation pairs.

    Summarizing the above: There exists a number $n^*$, such that with probability $\delta_{\dataset}$, a dataset consisting of $n^*$ histories contains at least $n_{s,z}$ samples for every state-observation pair $(s,z)$, where $n_{s,z}$ is the number computed using \cref{eq:okamoto} inserting $\eps$ and $\delta_{s,z}$.
    Consequently, $\Mhat$ using this dataset satisfies that with probability at least $\delta_{\Mhat}$, for all states $s\in S$ and observations $z\in Z$, we have $\Mhat(z \mid s) = M(z \mid s) \pm \eps$.
    Together, we can guarantee that probably (with probability at least $\delta = \delta_{\dataset} + \delta_{\Mhat}$), $\Mhat$ is approximately correct.

\end{proof}

\begin{proposition}
    The improvements described in \cref{sec:algoritm-AMCAR-AsMAR} for using knowledge retain the PAC guarantees stated in \cref{thm:alg-1}.
\end{proposition}
\begin{proof}
    The improvements use the fact that the underlying Bernoulli process in fact does not depend on all features in $\Ialw$. While it is correct to still split on these variables, obtaining two processes with the same true success probability, we can also merge them.

    More formally, observe that if feature $i$ does not affect the missingness probability of other features, for all valuations of feature $i$, the corresponding Bernoulli processes have the same success probability.
    MCAR missingness functions are the most extreme case of this, where the given state is completely irrelevant and it suffices to have one Bernoulli process per missingness indicator vector.
    As a side note: Observe that it is indeed necessary to consider every missingness indicator vector and not individual features, since the missingness probabilities need not be independent.
\end{proof}

\subsection[PAC guarantees for AIMI]{PAC guarantees for AIMI (\cref{sec:algorithm-AIMI})}\label{app:algo-alg2}

\algTwo*

\begin{proof}
    This proof is analogous to that of \cref{thm:alg-1}:
    every missingness probability computed by \cref{eq:Mhat-AIMI} corresponds to the empirical success probability of a Bernoulli process, which allows to apply the results from \cref{app:algo-bernoulli}.
    This proof differs in the argument why all states grouped together in the same Bernoulli process have the same success probability, and in the argument why it feasible to sample a dataset of the necessary size.

    By the independence assumption, we know that it suffices to learn every individual $\pr(\rand z_i \mid \rand z \sim M(s))$ for each $i\in I$.
    By non self-censoring, we know that this probability depends only on features in $I \setminus \{i\}$.
    Thus, the counter $\#(s,i,0)$ counts exactly the successes of a Bernoulli process with success probability $\pr(\rand z_i \mid \rand z \sim M(s))$, and $\#(s,i,1)$ counts the failures.

    It only remains to argue that a sufficient dataset can be feasibly obtained. 
    For this, we use the assumption that no feature is missing surely. 
    In other words, every feature has a positive probability to be observed. 
    Thus, every reachable states has a positive probability $m$ to be fully observed.
    Using this, we can repeat the argument from the proof of \cref{thm:alg-1}.
\end{proof}

\begin{proposition}
    The improvements described in \cref{sec:algoritm-AMCAR-AsMAR} for using knowledge retain the PAC guarantees stated in \cref{thm:alg-2}.
\end{proposition}
\begin{proof}
    (a) If we know from an m-graph that a particular feature $i$ is not influenced by feature $j$, for all valuations of $j$ the Bernoulli process has the same success probability. Thus, we can merge these Bernoulli processes and ignore feature $j$.
    
    (b) If we know the missingness function is simple MAR and feature $j$ goes missing, we know that it cannot influence the missingness probability of any other feature by definition~\citep{Mohan_graphical_2021}.
    Then, the proof is the same as in Case (a).

    (c) If the missingness function is MCAR, we know that no feature influences the missingness probability of any other feature. 
    Thus, we can repeatedly apply the argument of Case (a) to merge all Bernoulli processes until we have one for every feature. 
\end{proof}

\subsection[Computing epsilon-optimal policies]{Computing $\eps$-optimal policies (\cref{sec:alg-policy})}\label{app:algo-policies}

\policies*

\begin{proof}
    \textbf{Sampling the dataset.}
    We have sampling access with a fair policy, so every state has positive probability to be visited.
    Thus, for any finite number $n$, we can almost surely obtain $n$ samples of every state $s$ in finite time.
    For the Bernoulli process underlying \cref{eq:Mhat-AsMAR}, and if the missingness function is simple MAR, this suffices to guarantee that for every state-observation pair, we can obtain the number of samples $n_{s,z}$ required for achieving precision $\eps$ with confidence $\delta_{s,z}$.
    Similarly, for the Bernoulli process underlying \cref{eq:Mhat-AIMI}, and if the missingness function satisfies positivity, we can also obtain the required number of samples for every state-observation pair.
    Overall, under the assumptions of the theorem, we can almost surely obtain a dataset in finite time such that it suffices to give PAC guarantees on every state-observation pair.
    
    We remark that this does not even require spending confidence budget as we did in the proofs of \cref{thm:alg-1,thm:alg-2}, since there we required to get this dataset within a certain number of histories $n^*$.
    Here, we only claim that we can get a sufficient dataset in finite time almost surely.

    \textbf{Obtaining $\Mhat$.}
    The assumptions on the missingness function in the statement of the theorem match those in \cref{thm:alg-1} or \cref{thm:alg-2}.
    Hence, given the dataset described in the previous paragraph, we can approximate $\Mhat$ in a way such that with probability $\delta$, it is $\epsA$-precise.
    Note that here we do not employ the full allowed imprecision $\eps$, but rather a smaller $\epsA< \eps$, since there will be other sources of error.

    \textbf{$M$ and $\Mhat$ qualitatively agree.}
    For our technical reasoning, we require that $M(z\mid s) = 0$ if and only if $\Mhat(z\mid s) = 0$. 
    We prove both directions separately:
    If $M(z\mid s) = 0$, then we never observe a sample for $z$ when given $s$, and thus $\Mhat(z\mid s) = 0$, as it uses an empirical average (\cref{eq:Mhat-AIMI,eq:Mhat-AsMAR}).
    If $M(z\mid s) > 0$, as we use a fair sampling process, we almost surely eventually observe $z$ when given $s$, and consequently the empirical average is positive, i.e.\ $\Mhat(z\mid s) > 0$.
    
    It remains to prove that we can \emph{in finite time} conclude that $M$ and $\Mhat$ qualitatively agree.
    This means that we need to be sufficiently certain that if $\Mhat(z\mid s) = 0$, this is because indeed $M(z\mid s) = 0$ and not just because we haven't sampled enough yet.
    For this, we use a proof technique employed in, e.g., \cite{tacas16}:
    We utilize knowledge of (a lower bound on) the smallest missingness probability $p$.
    Further, recall that the confidence threshold $\delta$ is distributed over all Bernoulli processes (see \cref{app:algo-alg1,app:algo-alg2}).
    Thus, for each Bernoulli process, we have a confidence threshold $\delta_{s,z}$. 
    Okamoto's inequality (see \cref{app:algo-bernoulli}) provides an upper bound on the missingness probability that is correct with probability at least $\delta_{s,z}$.
    Thus, when this upper bound is less than $p$, we can conclude with sufficient confidence that $\Mhat(z\mid s) = 0$.

    \textbf{Utilizing \cref{lem:policy-val-diff}.}
    Let $\mdphat$ be the approximated missingness-MDP that is exactly $\mdp$ except for the missingness function, which is $\Mhat$ instead of $M$.
    We have just proven that in finite time we know that with probability $\delta$, $\Mhat$ is $\epsA$-precise and qualitatively agrees with $M$.
    Thus, it satisfies the assumptions specified in \cref{lem:policy-val-diff}, which is proven below.
    This key technical lemma shows that the values obtained when following a policy $\pi$ in either the original $\mdp$ or the approximated $\mdphat$ have a bounded difference.\footnote{We highlight that every policy is applicable in both missingness-MDPs, as they only differ in their missingness probabilities, but agree on states, observations, and actions.}
    Formally, for every policy $\pi$, we have $\abs{V_{\mdp}(\pi) -V_{\mdphat}(\pi)} \leq f(\epsA)$, where $f$ is a monotonically increasing function that depends on $\epsA$, the precision of $\Mhat$.
    
    From this, we obtain two facts: 
    Firstly, since this holds for all policies, it also holds for the supremum over all policies, and thus we can bound the difference in the values of the two missingness-MDPs:
    \begin{equation}
        \abs{\sup_{\pi} V_{\mdp}(\pi) - \sup_{\pi} V_{{\mdphat}}(\pi)} \leq f(\epsA).\label{eq:proof-policy-valdiff}
    \end{equation}
    Secondly, we can apply the same reasoning to a near-optimal policy in $\mdphat$. 
    For this, let $\epsPi < \eps$ be a precision smaller than our overall error tolerance, and let $\pi^*$ be an $\epsPi$-optimal policy in $\mdphat$, i.e.
    \begin{equation}
        \sup_{\pi} (V_{\mdphat}(\pi)) - V_{{\mdphat}}(\pi^*) \leq \epsPi.\label{eq:proof-policy-eps-opt}
    \end{equation}
    We remark that $\mdphat$ is a fully specified missingness-MDP, and thus a fully specified POMDP, for which solvers computing $\eps$-optimal policies such as SARSOP~\citep{kurniawati2008sarsop} exist.
    Using \cref{lem:policy-val-diff}, we obtain the following inequality:
    \begin{equation}
        \abs{V_{\mdp}(\pi^*) - V_{\mdphat}(\pi^*)} \leq f(\epsA).\label{eq:proof-policy-optdiff}
    \end{equation}

    \textbf{Implications of the inequalities.}
    Since we reason about absolute differences, we need to make case distinctions on whether $\sup_{\pi} V_{\mdp}(\pi) - \sup_{\pi} V_{{\mdphat}}(\pi) \geq 0$ or not when applying \cref{eq:proof-policy-valdiff}.
    If $\sup_{\pi} V_{\mdp}(\pi) - \sup_{\pi} V_{{\mdphat}}(\pi) \geq 0$, then $\sup_{\pi} V_{\mdp}(\pi) - \sup_{\pi} V_{{\mdphat}}(\pi) \leq f(\epsA)$, and by reordering we get $\sup_{\pi} V_{\mdp}(\pi) \leq \sup_{\pi} V_{{\mdphat}}(\pi) + f(\epsA)$.
    Otherwise, we have $\sup_{\pi} V_{\mdp}(\pi) < \sup_{\pi} V_{{\mdphat}}(\pi)$.
    Together, we can obtain that \cref{eq:proof-policy-valdiff} implies:
    \begin{equation}
        \sup_{\pi} V_{\mdp}(\pi) \leq \sup_{\pi} V_{{\mdphat}}(\pi) + f(\epsA) \label{eq:proof-policy-valdiff-mod}
    \end{equation}
    Analogously, we can make a case distinction in \cref{eq:proof-policy-optdiff} and obtain that:
    \begin{equation}
        V_{\mdphat}(\pi^*) \leq V_{\mdp}(\pi^*) + f(\epsA) \label{eq:proof-policy-optdiff-mod}
    \end{equation}

    \textbf{Combining the inequalities.}
    To conclude the proof, we use a chain of inequalities.
    \begin{align*}
        \sup_{\pi} V_{\mdp}(\pi) &\leq \sup_{\pi} V_{{\mdphat}}(\pi) + f(\epsA) \tag{By \cref{eq:proof-policy-valdiff-mod}}\\
        &\leq V_{{\mdphat}}(\pi^*) + \epsPi + f(\epsA) \tag{By \cref{eq:proof-policy-eps-opt}}\\
        &\leq V_{\mdp}(\pi^*) + f(\epsA) + \epsPi + f(\epsA) \tag{By \cref{eq:proof-policy-optdiff-mod}}
    \end{align*}
    By reordering, we obtain
    \[
        \abs{\sup_{\pi} V_{\mdp}(\pi) - V_{{\mdp}}(\pi^*)} \leq \epsPi + 2\cdot f(\epsA).
    \]
    Hence, since $f$ is a monotonically increasing function, there exists a choice of $\epsA$ and $\epsPi$ so that $\epsPi + 2\cdot f(\epsA)  \leq \eps$.
    Intuitively, while the errors incurred by approximating $\Mhat$ and by using an approximately optimal policy add up, we can bound the overall maximum error.
    Thus, we can choose the two precisions so that the overall error criterion is met, and the policy $\pi^*$ is $\eps$-optimal in the original missingness-MDP (with probability $\delta$; with the remaining probability, our sampling was unlucky and $\Mhat$ can differ by more than $\epsA$).
    This concludes the proof.
\end{proof}

\begin{lemma}[Bounding the Value-Difference between $\mdp$ and $\mdphat$]\label{lem:policy-val-diff}
    Let $\mdp$ be a missingness-MDP and $\mdphat$ be a missingness-MDP that differs from $\mdp$ only in its missingness function, where it uses $\Mhat$ instead of $M$.
    Further, assume that for all states $s\in S$ and observations $z\in Z$, we have $M(z\mid s) = 0$ if and only if $\Mhat(z\mid s) = 0$, and moreover $M(z\mid s) = \Mhat(z\mid s) \pm \epsA$.
    Then, for every policy $\pi$ we have 
    $\abs{V_{\mdp}(\pi) -V_{\mdphat}(\pi)} \leq f(\epsA)$, where $f$ is a monotonically increasing function.
\end{lemma}
\begin{proof}
    \textbf{To uncountable MDPs.}
    Note that both $\mdp$ and $\mdphat$ are missingness-MDPs, and thus POMDPs.
    Thus, for each of them, we can construct an uncountable belief MDP with the same value, called $B$ or $\widehat{B}$, respectively.
    Intuitively, this is achieved by unrolling step-by-step the observation function and all possible beliefs that the agent can have after an action; the transition probabilities in these uncountable MDPs depend on the missingness functions.
    For a more extensive description, see \mbox{\cite[Chapter 16.4.1]{AIMA}}.

    \textbf{To finite MDPs.}
    We consider discounted expected reward, with $\gamma$ the discount factor and $\reward_{\max} \coloneqq \max_{(s,a)\in S\times A}\reward(s,a)$ the maximum state reward.
    As the expected reward is a geometric series, we can bound the reward that can be obtained after $n$ steps from above as follows:
    \[
        \sum_{i=n}^\infty \gamma^i \cdot \reward_{\max} = \gamma^n \cdot \reward_{\max} \cdot \sum_{i=0}^\infty \gamma^i = \frac{\gamma^n \cdot \reward_{\max}}{1-\gamma}.
    \]
    For every arbitrarily small precision $\epsG>0$, we can thus obtain an $n$ such that the reward after $n$ steps is less than $\epsG$.
    Let $B_{\epsG}$ be the finite MDP obtained from $B$ by only considering states that are reachable within $n$ steps, and analogously define $\widehat{B}_{\epsG}$.
    (Note that $n$ is the same for both, since it only depends on $\gamma$ and $\reward_{\max}$, which is the same for both of them.)
    The value of these finite belief MDPs differs from the value of the uncountable belief MDPs and thus the original missingness-MDPs by at most $\epsG$.

    \begin{figure*}[t]
    \centering
    \begin{subfigure}[b]{0.45\textwidth}
        \includegraphics[width=\linewidth]{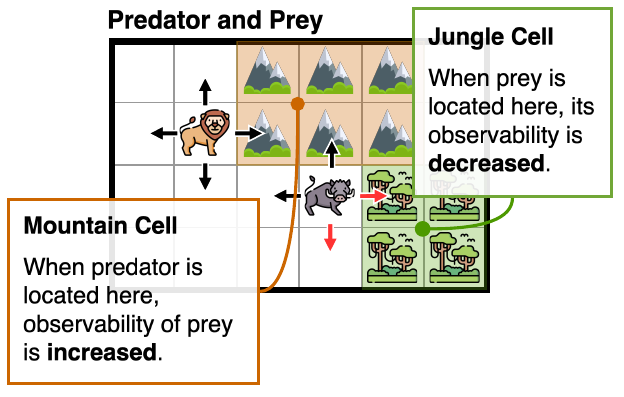}
        \label{fig:predator-prey-visualization}
    \end{subfigure}
    \hfill
    \begin{subfigure}[b]{0.35\textwidth}
        \includegraphics[width=\linewidth]{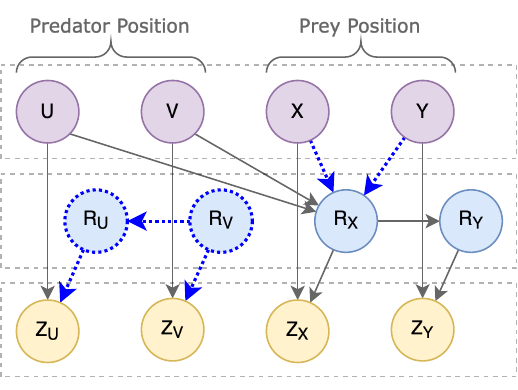}
        \label{fig:predator-prey-m-graph}
    \end{subfigure}
    \hfill
    \caption{\textbf{left:} The \Predator benchmark, where the predator (lion) is the agent trying to catch its prey (boar). Predator and prey can move in all four cardinal directions, where prey chooses an action that increases the distance to the predator (red arrows). \textbf{right:} The m-graphs for the predator and prey benchmark describing missingness functions of types \emph{simple MAR} (\textcolor{gray}{gray}), \emph{identifiable MNAR} (\textcolor{gray}{gray} + \textcolor{blue}{blue}). Causal dependencies between the state features were omitted for clarity.}
    \label{fig:predator-prey-ap}
    
    \end{figure*}

    \textbf{Bounding the difference.}
    Recall that $B$ or $\widehat{B}$ are the same except for their transition functions, which depend on $M$ and $\Mhat$, respectively.
    Still, by assumption of the theorem $M$ and $\Mhat$ qualitatively agree, i.e.\ $M(z\mid s) = 0$ if and only if $\Mhat(z\mid s) = 0$.
    Hence, the graph structure of $B$ or $\widehat{B}$ is the same. 
    Thus, the only difference are small perturbations of individual transition probabilities by at most $\epsA$.
    
    It remains to show the following: Given two finite MDPs that are the same except for small perturbations of the transition probabilities, but where the supports of the the transition functions are the same, provide a bound on the difference in their value.
    Such a result exists in the literature, namely in~\cite{MWW25}, or more precisely in the extended version of that paper~\cite[Lemma 5]{MWW25arxiv}.
    It remains to show that our setting indeed satisfies the assumptions of \cite[Lemma 5]{MWW25arxiv}.

    \begin{itemize}[leftmargin=20pt, topsep=0pt, itemsep=0pt]
        \item \enquote{For every closed constant-support RMDP}: Their claim applies to robust MDPs that are closed constant-support. 
        A robust MDP is an MDP whose transitions are not probability distributions, but rather sets of possible values, see ~\cite[Section 2]{MWW25}. 
        In our case, instead of considering the concrete MDPs $B_{\epsG}$ and $\widehat{B}_{\epsG}$, we consider the robust MDP that arises when considering an $\epsA$-interval around every missingness probability $M(z\mid s)$.
        This robust MDP contains both $B_{\epsG}$ and $\widehat{B}_{\epsG}$ as instantiations.
        \item \enquote{For every pair of agent and environment policy}: An agent policy in this setting is exactly the agent policy in ours, so ~\cite[Lemma 5]{MWW25arxiv} applies to all policies.
        An environment policy is the policy that chooses the instantiation of the transition function, i.e.\ the exact missingness probabilities from the set of all that differ by at most $\epsA$ in our setting.
        \item \enquote{Total-reward objectives:} \cite[Lemma 5]{MWW25arxiv} concerns \emph{undiscounted} total-reward or mean payoff objectives. Undiscounted total-reward generalizes discounted expected reward, using the standard construction which adds an edge transitioning with probability $\gamma$ to a dedicated sink state to every transition. Thus, the lemma is applicable to the objective in our setting.
        \item \enquote{The value function is continuous w.r.t. the environment policy}: This is the claim of~\cite[Lemma 5]{MWW25arxiv}. More formally, if the environment chooses missingness probabilities differently with some deviation $\epsA$, then the deviation in the value between the two instantiations is bounded by some monotonically increasing function $g(\epsA)$.
        This is exactly the claim we require, since it means that for all agent policies $\pi$ and all missingness functions $\Mhat$ that are $\epsA$-close to $M$, we have $\abs{V_{B_{\epsG}}(\pi) -V_{\widehat{B}_{\epsG}}(\pi)} \leq g(\epsA)$.
        
        We also argue that $g$ can be effectively computed, as it depends on the size of the state space, the reward function, and the minimum occurring transition probability, all of which are known to us (recall that \cref{thm:policies} assumes knowledge of a lower bound on the minimum missingness probabilities).
        The concrete way of deriving the distance is provided on~\cite[page 17]{MWW25arxiv}.
    \end{itemize}

    \textbf{Putting it all together.}
    Our goal is to show that we can compute an $f$ such that for all policies $\pi$ we have: $\abs{V_{\mdp}(\pi) -V_{\mdphat}(\pi)} \leq f(\epsA)$.
    The following chain of equations proves our goal:
    \begin{align*}
        \abs{V_{\mdp}(\pi) &- V_{\mdphat}(\pi)} = 
        \abs{V_{B}(\pi) -V_{\widehat{B}}(\pi)}\\ 
        &\text{(Using the uncountable belief MDPs)}\\
        &\leq \abs{B_{\epsG}(\pi) -V_{\widehat{B}_{\epsG}}(\pi)} + \epsG\\ 
        &\text{(Using the finite MDPs; decreasing both values by}\\
        &\text{at most $\epsG$ increases the difference by at most $\epsG$)}\\
        &\leq g(\epsA) + \epsG\\ 
        &\text{(By bounding the difference).}
    \end{align*}

    For simplicity of presentation, we choose $\epsG = \epsA$, and thus setting $f(\epsA) \coloneqq g(\epsA) + \epsA$ concludes the proof.    
\end{proof}

\section{Benchmarks and Experiments}
\label{app:benchmarks}

Here we describe our benchmarks. 
A benchmark corresponds to an environment with a \emph{particular} missingness function.
We provide a detailed description of the benchmarks as well as the parameters for running the experiments.

\begin{figure}[tbh]
    \centering
    \includegraphics[width=0.7\linewidth]{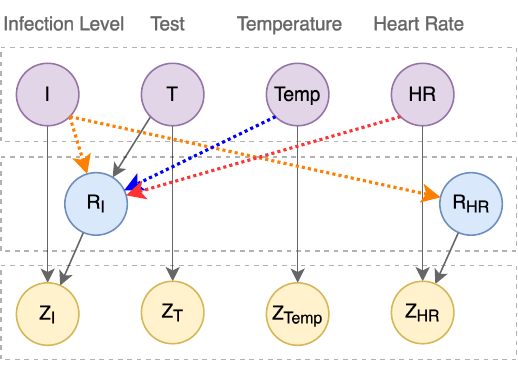}
    \caption{The m-graphs for the ICU benchmark describing missingness functions of types \emph{simple MAR} (\textcolor{gray}{gray} + \textcolor{blue}{blue}), \emph{identifiable MNAR} (\textcolor{gray}{gray} + \textcolor{red}{red}) and \emph{unidentifiable MNAR} (\textcolor{gray}{gray} + \textcolor{red}{red} + \textcolor{orange}{orange}). Causal dependencies between the state features were omitted for clarity.}
    \label{fig:ICU-m-graph}
\end{figure}

\begin{table*}[tb]
\centering
\caption{Hyper parameter distributions for the random search for PPO policies.}
\label{tab:parameter-search}
\begin{tabular}{@{}llcc@{}}
\toprule
\textbf{Category}  & \textbf{Hyperparameter}  & \textbf{Search space}     & \textbf{Sampling} \\ \midrule
Architecture       & Number of layers         & \{2, 3, 4, 5\}            & Discrete uniform  \\
Architecture       & Hidden units per layer   & \{32, 64, 128, 256, 512\} & Discrete uniform  \\
PPO                & Clip range ($\epsilon$)           & {[}0.10, 0.30{]}          & Uniform           \\
PPO                & Entropy coefficient ($\lambda_\epsilon$) & {[}1e-4, 5e-2{]}          & Log-uniform       \\
PPO                & Target KL divergence     & {[}5e-3, 3e-2{]}          & Log-uniform       \\
Optimizer (actor)  & Learning rate            & {[}1e-5, 3e-4{]}          & Log-uniform       \\
Optimizer (critic) & Learning rate            & {[}1e-4, 3e-3{]}          & Log-uniform       \\
Optimization       & Minibatch size           & \{32, 64, 128, 256\}      & Discrete uniform  \\
Optimization       & Epochs per update        & \{5, 10, 20\}             & Discrete uniform  \\ \bottomrule
\end{tabular}
\end{table*}

\subsection{Description}

\begin{table}[tb]
\centering
\caption{Sizes of the feature dimensions of each environment. Sizes for the MCAR missingness function in the Predator environment in parenthesis. }
\label{tab:model-feature-sizes}
\begin{tabular}{@{}ll|ll@{}}
\toprule
\multicolumn{2}{c|}{\textbf{ICU}} & \multicolumn{2}{c}{\textbf{Predator}} \\
Feature             & Size        & Feature            & Size             \\ \midrule
Infection           & 4           & u                  & 5 (10)           \\
Test                & 5           & v                  & 5 (5)            \\
Temperature         & 4           & x                  & 5 (10)           \\
Hear Rate           & 10          & y                  & 5 (5)            \\ \bottomrule
\end{tabular}
\end{table}

\paragraph{\ICU.}
This environment, inspired by prior clinical decision-making models~\citep{johnson_alistair_mimic-iv_2022, pollard_eicu_2018, thoral_sharing_2021, hyland_early_2020}, simulates a doctor treating a patient with an infection that progresses stochastically over time.  
The state of the patient consists of the \emph{infection severity}, the \emph{temperature}, and the \emph{heart rate}.
The infection causally influences both the heart rate and the temperature.

The doctor has an option to wait, to administer costly antibiotics that reduce the infection severity, or to order a test, which is a measuring action that may reveal the infection severity.
The reward function penalizes high infection levels as well as costly interventions (ordering a test and administering antibiotics).
Thus, the doctor's objective is to maintain the patient's infection severity at low levels by administering antibiotics only when necessary.
For ease of modeling, the state space also includes the value of the last test ordered.

We evaluate three different missingness functions $M$, corresponding to distinct missingness functions, illustrated in the m-graph in \cref{fig:ICU-m-graph}.  
In all cases, the heart rate and the infection severity may be missing, whereas temperature and the last test ordered are always observed.
The success rate of the test that reveals the infection severity may depend on different features, resulting in the following missingness functions.
\textbf{(1)} \textbf{Simple MAR}, where the success rate only depends on the (always observed) temperature.
\textbf{(2)} \textbf{MNAR (id.)}, where the success rate only depends on the (not always observed) heart rate, resulting in an identifiable MNAR function without self-censoring and satisfying the positivity assumption.
\textbf{(3)} \textbf{MNAR (unid.)} is an extension of \textbf{MNAR (id.)}, where the infection severity influences the test success rate, introducing self-censoring and thus making the function unidentifiable.

The ICU environment has 800 reachable states and 1650 observations for all missingness functions. More details on feature dimensions are given in \cref{tab:model-feature-sizes}.

\begin{figure*}[tbh]
    \centering
    \includegraphics[width=0.75\linewidth]{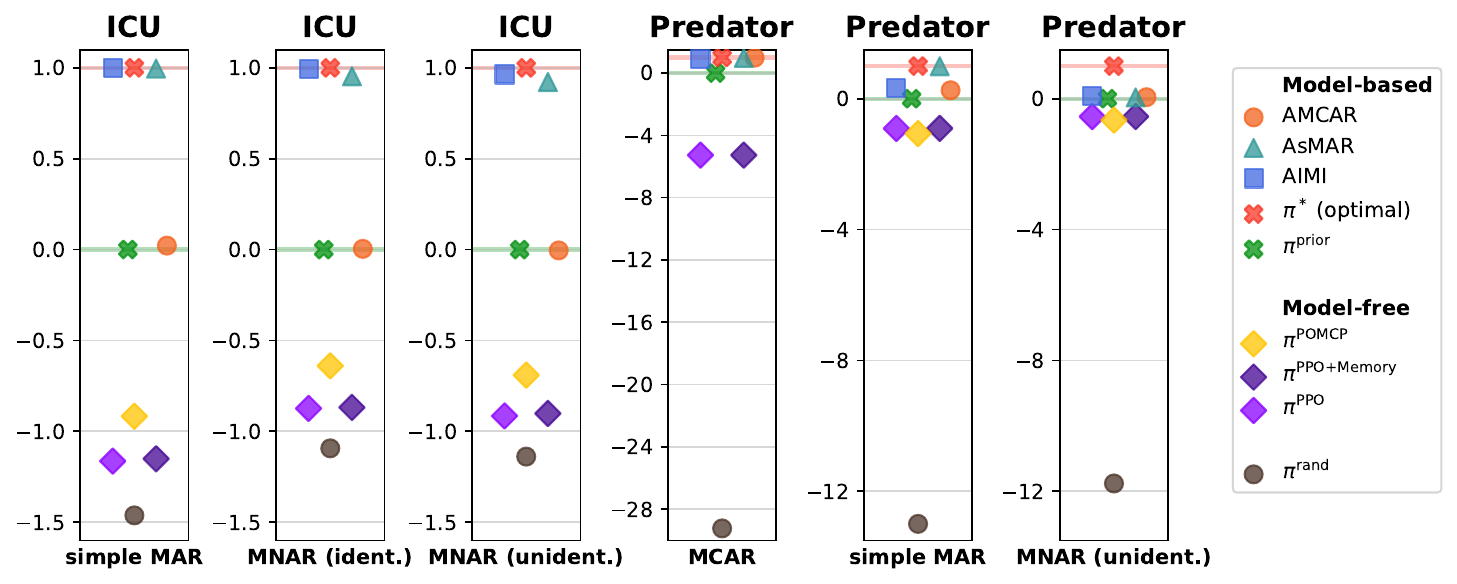}
    \caption{Baseline comparison plot for all benchmarks.}
    \label{fig:baseline-comparison-all}
\end{figure*}

\paragraph{\Predator.}
This environment is a variant of the \emph{Tag} benchmark from~\cite{DBLP:conf/ijcai/PineauGT03}, where an agent (in our case, a predator) is tasked with chasing a partially hidden target (a prey) in a 2D grid environment. 
The prey senses the predator and usually moves away from it; in case multiple directions lead away from the predator, the prey chooses uniformly at random. 
The predator's movement is deterministic (dictated by the policy), but moving in an intended direction may randomly fail due to terrain conditions. 
Predator obtains a flat reward upon catching the prey, and thus the discounting incentivizes catching the prey as soon as possible. 

The environment may feature three distinct \emph{biomes} -- plains, mountains, or jungles -- that influence the predator's observability of the prey, see \cref{fig:predator-prey-ap}, and thus define the missingness function. 
We investigate the following three variants thereof.
\textbf{(1) MCAR}, which features only one type of terrain, i.e., the prey is observed with constant probability.
\textbf{(2) simple MAR}, where the environment features plains as well as mountains from which the predator has a higher chance of observing its target. 
\textbf{(3) MNAR (unid.)}, where the prey has an option to hide in jungle cells, introducing self-censoring of its position.
We stress that when the predator loses track of the prey, both features corresponding to $x$ and $y$ coordinates of the prey go missing simultaneously, modeled by dependencies between missingness indicators $R_x$ and $R_y$.
The dependence between the missingness indicators is one key difference from the ICU benchmark together with lower stochasticity in the transition function.

The Predator environment has 626 reachable states and 1311 observations for the simple MAR and MNAR (unid.) missingness functions, including one initial state. For the MCAR missingness function we used a larger map resulting in a model with 1765 reachable states and 3379 observations. More details on feature dimensions are given in \cref{tab:model-feature-sizes}.

\subsection{Experimental setup}\label{app:setup}

\paragraph{Technical Setup.}
For SARSOP and POMCP, we used high-performance workstations equipped with an AMD Ryzen ThreadRipper PRO 5965WX (24-core, 3.8GHz) CPU, 512 GB ECC DDR4 RAM,
and a 2 TB PCIe 4.0 NVMe SSD.
The PPO benchmarks were obtained on a Macbook Pro, with an M3 chip with 11 cores, 18 GB LPDDR5 RAM.

\paragraph{Simulating trajectories.}
For both benchmarks, we used a discount factor of $\gamma = 0.95$.
We considered dataset sizes $\datasetsize \in \{10,50,100,500,10^3,5 \cdot 10^3,10^4,10^5,10^6,10^7\}$.
To obtain a dataset containing \datasetsize samples, we simulated finite trajectories until their lengths summed up to \datasetsize.
A trajectory is terminated when it reaches a terminal state (only for the \Predator bechmark, when the predator catches the prey) or if its length exceeds $L = \left\lceil \log_{\gamma}\frac{(1-\gamma) \cdot 10^{-3}}{\reward_{\max}} \right\rceil$, where $\reward_{\max} \coloneqq \max_{s,a}\reward(s,a)$. Here, $L$ denotes the smallest integer satisfying $\sum_{k=L}^\infty \gamma^k \cdot \reward_{\max} < 10^{-3}$, i.e. a time step after which the maximum discounted cumulative reward cannot exceed $10^{-3}$.
For each dataset size \datasetsize, we generated 20 independent datasets of this size.

\paragraph{Timeouts \& precision.}
For the baselines, we used the timeout of 5 minutes when solving the POMDP (to obtain \piopt and \piguess) and the same timeout to evaluate the resulting policy (or \pirand).
To obtain a policy \pihat by solving the corresponding $\MDPhat$, we used a timeout of 3 minutes and evaluated \pihat for 2 minutes. In all cases, solving was additionally allowed to terminate upon reaching the relative precision of $10^{-3}$.

\subsection{PPO Baseline}

The baseline policies for Proximal Policy Optimization (PPO) \citep{DBLP:journals/corr/SchulmanWDRK17} were computed using the implementation of the \texttt{Crux.jl} framework (Version 0.1.3) in Julia (Version 1.11.5).

\paragraph{Frame Stacking.}
In order to provide \emph{memory} to the PPO policy, 
the last $n_{\text{frames}}$ observations and actions were concatenated.
At an episode start the missing frames in the history were padded with POMDP's symbol for missing values.
The number of frames was determined by $n_\text{frames} = \left\lceil \log_\gamma(\varepsilon(1-\gamma))\right\rceil$, where $\gamma$ is the respective \missMDP's discount factor and $\varepsilon$ is the same precision used for evaluating policies computed via SARSOP.
With a history of $n_\text{frames}$ the policy had access to a history sufficient for accurate evaluation, as stated in the next paragraph.  
For ICU and Predator the frames were set to $n_\text{frames}=148$.

\paragraph{Training and Evaluation.}
All PPO policies were trained for $n_\text{steps}= 10^6$ steps with individual seeds using the Adam optimizer.
The policies were evaluated using cumulative discounted return, identical to the evaluation of the other benchmarks.
During evaluation actions were sampled from PPO's action distribution.
The mean performance was computed over $n_\text{eval}= 2000$ episodes.
The evaluation horizon per episode was set to $n_\text{frames}$ (see above).

\paragraph{Parameter Search and Model Selection.}
We performed 20 random searches for the hyper parameters, for each benchmark and corresponding missingness function (20 searches for each \ICUSimpleMAR, \ICUMNARid, \ICUMNARunid, \ldots). 
The policy and value networks always used the same architectural hyper parameters.
We report the performance of the best model for each search with respect to the cumulative discounted reward.
\cref{tab:parameter-search} provides the parameter distributions of the random search.
The hyper parameters for all runs on each benchmark are provided in our repository.

\subsection{POMCP Baseline}

The baseline policies for Partially Observable Monte Carlo Planning (POMCP) \citep{silver2010monte} were computed using the \texttt{BasicPOMCP.jl} framework (Version 0.3.12) in Julia (Version 1.11.5), which implements the PO-UCT online tree search algorithm together with particle filters (resulting in POMCP) for POMDPs.

\paragraph{Planning Setup.}
POMCP was used as an online, planning baseline with access to the transition and reward functions of the underlying \missMDP.
Belief updates were performed via an external belief updater, as \texttt{BasicPOMCP.jl} does not reuse particles from planning simulations for belief propagation.
In fact, in order to make POMCP competitive against the model-based SARSOP policies, it was provided the \emph{true} belief whenever its own belief collapsed due to an unseen observation.
At each decision step, planning was performed from scratch based on the current belief.

\paragraph{Solver Configuration.}
Leaf values were estimated using rollout-based evaluation with a random rollout policy.
Unless stated otherwise, we used a maximum planning depth of 20, the default exploration constant of $c=1.0$, and $1000$ tree queries per action.
Random seeds were fixed to ensure reproducibility.

\paragraph{Evaluation Protocol.}
POMCP policies were evaluated using cumulative discounted return, consistent with all other baselines.
During evaluation, the planner was executed online and selected actions based on the current belief.
Results were averaged over $n_\text{eval}=2000$ episodes, using the same discount factors and episode horizons as in the remaining benchmarks.

\subsection{Additional Results}

\begin{figure}[t]
  \centering
  \begin{minipage}[c]{0.65\linewidth}
    \centering
    \includegraphics[width=\linewidth]{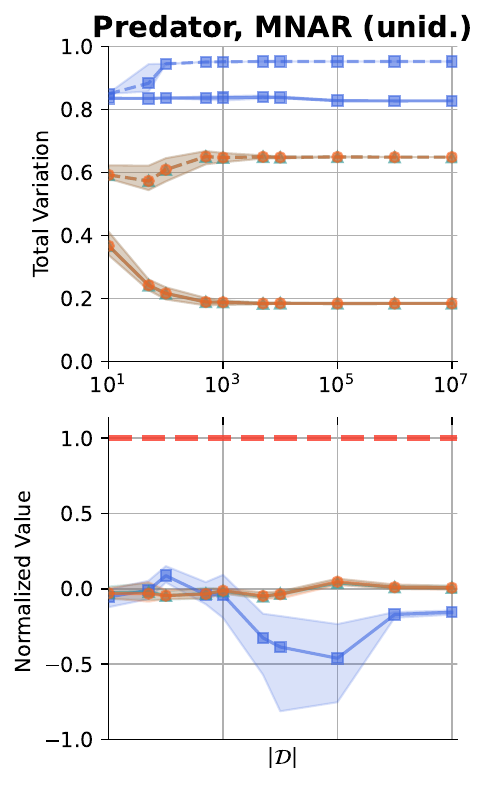}
  \end{minipage}\hfill
  \begin{minipage}[c]{0.34\linewidth}
    \centering
    \raisebox{0pt}[\height][\depth]{%
      \includegraphics[width=\linewidth]{legend_vertical.pdf}%
    }
  \end{minipage}
  \caption{Empirical results for the \PredatorMNARunid benchmark.}
  \label{fig:predator-mnar-unid}
\end{figure}

\paragraph{}
We provide results for an additional \PredatorMNARunid benchmark in \cref{fig:predator-mnar-unid}.
Its missingness function breaks the assumptions of our algorithms \AMCAR, \AsMAR and \AIMI.
Therefore, all of the approximations of the missingness functions $\Mhat$ converge to a non-zero total variation.
As a result, the policies computed by SARSOP using $\Mhat$ do not converge to the optimal policy (\piopt).
Instead all policies perform similarly to \piguess, the policy initialized with an uninformed prior on $M$ and corresponds to the normalized value of zero.

In \cref{fig:baseline-comparison-all} we provide the comparison with baselines for all benchmarks, adding to \cref{fig:baseline-comparison}.

\FloatBarrier

\end{document}